\documentclass{article}

\usepackage[accepted]{icml2025}

\usepackage[toc,page,header]{appendix}
\usepackage{minitoc}

\usepackage{microtype}
\usepackage{graphicx}
\usepackage{subfigure}
\usepackage{booktabs} 
\usepackage{tcolorbox}
\usepackage{float}

\usepackage{hyperref}

\usepackage{algorithm}
\usepackage{algorithmic}

\usepackage{makecell}
\usepackage{booktabs}       
\usepackage{color}
\usepackage{amsfonts}       
\usepackage{nicefrac}       
\usepackage{enumerate}   
\usepackage{bm}
\usepackage{microtype}      
\usepackage{xcolor}         
\usepackage{bbm}
\usepackage{xfrac}
\usepackage{tocbibind}
\usepackage{tikz}
\usepackage{physics}
\usepackage{hyperref}       
\usepackage{dsfont}
\usepackage{graphicx}
\usepackage{url} 

\usepackage{caption}

\usepackage{listofitems} 
\usepackage{subfiles}
\usepackage[
  separate-uncertainty = true,
  multi-part-units = repeat
]{siunitx}

\usepackage{enumitem}
\usepackage{arydshln}
\usepackage{multirow}

\usepackage{hyperref}

\usepackage{amsmath}
\usepackage{amssymb}
\usepackage{mathtools}
\usepackage{amsthm}
\usepackage[capitalize,noabbrev]{cleveref}

\theoremstyle{plain}
\newtheorem{theorem}{Theorem}[section]
\newtheorem{proposition}[theorem]{Proposition}
\newtheorem{lemma}[theorem]{Lemma}

\theoremstyle{definition}

\newtheorem{assumption}[theorem]{Assumption}
\theoremstyle{remark}

\theoremstyle{plain}
\theoremstyle{definition}
\theoremstyle{remark}

\makeatletter
\newtheorem*{rep@theorem}{\rep@title}
\newcommand{\newreptheorem}[2]{%
\newenvironment{rep#1}[1]{%
 \def\rep@title{#2 \ref{##1}}%
 \begin{rep@theorem}}%
 {\end{rep@theorem}}}
\makeatother

\newreptheorem{theorem}{Theorem}
\newreptheorem{lemma}{Lemma}
\newreptheorem{proposition}{Proposition}
\newreptheorem{corollary}{Corollary}
\newreptheorem{assumption}{Assumption}

\newcommand{\mrd}{\mathrm{d}}
\newcommand{\mbX}{\mathbf{X}}

\newcommand{\bbR}{\mathbb{R}}
\newcommand{\mrR}{\mathrm{R}}
\newcommand{\mbA}{\mathbf{A}}

\newcommand{\KLr}{\mathrm{KL}}
\newcommand{\mbR}{\mathbb{R}}
\newcommand{\mbZ}{\mathbf{Z}}

\newcommand{\degmax}{\mathrm{deg}^{\max}}
\newcommand{\degmin}{\mathrm{deg}^{\min}}

\newcommand{\bx}{\mathbf{x}}

\newcommand{\Hcal}{\mathcal{H}}
\newcommand{\Rcal}{\mathcal{R}}

\newcommand{\E}{\mathbb{E}}

\usepackage{xspace}
\usepackage[colorinlistoftodos, color=blue!30!white 
]{todonotes}                                        

\icmltitlerunning{Imbalance Transductive Node Classification}

\begin{document}
\doparttoc 
\faketableofcontents 

\twocolumn[
\icmltitle{UPL: Uncertainty-aware Pseudo-labeling \\ for
Imbalance Transductive Node Classification}



\icmlsetsymbol{equal}{*}
\begin{icmlauthorlist}
\icmlauthor{Mohammad T. Teimuri}{sharif}
\icmlauthor{Zahra Dehghanian}{sharif}
\icmlauthor{Gholamali Aminian}{turing}
\icmlauthor{Hamid R. Rabiee}{sharif}
\end{icmlauthorlist}

\icmlaffiliation{sharif}{Sharif University of Technology}
\icmlaffiliation{turing}{Alan Turing Institute}

\icmlcorrespondingauthor{Hamid R. Rabiee}{rabiee@sharif.edu}




\vskip 0.3in
]

\begin{abstract}
Graph-structured datasets often suffer from class imbalance, which complicates node classification tasks. In this work, we address this issue by first providing an upper bound on population risk for imbalanced transductive node classification. We then propose a simple and novel algorithm, Uncertainty-aware Pseudo-labeling  (UPL). Our approach leverages pseudo-labels assigned to unlabeled nodes to mitigate the adverse effects of imbalance on classification accuracy. Furthermore, the UPL algorithm enhances the accuracy of pseudo-labeling by reducing training noise of pseudo-labels through a novel uncertainty-aware approach. We comprehensively evaluate the UPL algorithm across various benchmark datasets, demonstrating its superior performance compared to existing state-of-the-art methods.
\end{abstract}
\vspace{-.2in}
\section{Introduction}
Graphs are fundamental data structures that are universally utilized in various applications. The flexibility and expressiveness of graph representations enable us to model complex relationships and interactions within several types of data across numerous fields. From social networks, where individuals are connected by various kinds of relationships, to biological networks that represent protein interactions, graphs provide a powerful way to encapsulate the inherent complexities of many natural and human-made systems \citep{mohammadrezaei2018identifying, ying2018graph, perozzi2016recommendation, hamilton2017inductive}.
In tackling the task of learning from graph data, Graph Neural Networks (GNNs) play an important role \citep{hamilton2017inductive, kipf2016semi, velickovic2017graph}. GNNs leverage the rich relational information within graphs, allowing for the effective propagation of information through the network structure. And their ability to generalize deep learning approaches to non-Euclidean data. 
However, despite their advantages, GNNs encounter significant challenges when dealing with imbalanced data in transductive scenarios where the unlabeled nodes are utilized during the training phase, a common scenario in real-world graphs \citep{park2022graphens}. In many graph-based applications, some classes of nodes are vastly underrepresented compared to others, leading to a problem in the learning process \citep{song2022tam, park2022graphens}. In particular, this imbalance can result in biased models towards the majority class, which can also depend on graph structure. For instance, in social networks, minority groups might be underrepresented in friendship suggestions; rare protein interactions could be overlooked in biological networks. 

The problem of imbalance in graphs thus poses unique challenges that require specialized strategies to address. Standard imbalance handling techniques, such as re-weighting and re-sampling, may work in fields such as image classification but face significant obstacles when applied to graphs. Re-weighting methods, which assign higher weights to minor nodes, do not alter the connectivity patterns of these nodes, leaving their neighborhood influences unchanged. Re-sampling strategies, intended to balance class representation, can disrupt the graph's structure, potentially altering critical message-passing pathways and affecting the semantic integrity of node classes and embedded features.

Despite these developments, balancing the enhancement of minority class representation against the integrity of majority classes remains a technical challenge due to overemphasis on minority classes can degrade the model’s ability to accurately represent more common classes \citep{yan2024rethinking, song2022tam, park2022graphens,jervakani2024klce}. Furthermore, the generalization error and population risk study of imbalanced node classification is overlooked. Our contributions in this work can be summarized as follows:
\vspace{-.1in}
\begin{itemize}[noitemsep,nolistsep,leftmargin=*]
    \item  We study the population risk of transductive node classification under an imbalance scenario and derive an upper bound on population risk which depends on different graph properties of each class in the graph.
    \item We present a novel algorithm, uncertainty-aware pseudo-labeling that demonstrates superior performance compared to existing approaches in many homophilic and heterophilic graph datasets. 
\end{itemize}

\section{Related works}

\textbf{Imbalance Graph Classification:}
Due to the non-Euclidean characteristics of graph data, the imbalance problem in this type of data necessitates additional attention and specialized focus. Following the categorization of imbalance classification methodologies, we similarly categorize imbalance graph classification models into two main groups: data-level and algorithm-level methods.

\textit{Data-level methodologies} try to rebalance the learning environment by manipulating training data in feature or label spaces. Fundamental approaches such as over-sampling and under-sampling, prevalent in traditional class-imbalanced learning, require adaptation to suit the complexities inherent in graph data, characterized by intricate node dependencies and interconnections. Data-level methods could subcategorized into data interpolation, adversarial generation, and pseudo-labeling \citep{ma2023class}. Data interpolation techniques, exemplified by SMOTE \citep{chawla2002smote} and its graph-oriented variant, GraphSMOTE \citep{zhao2021graphsmote}, generate synthetic minority nodes by interpolating between existing node representations. Methods like GATSMOTE \citep{liu2022gatsmote} and GNN-CL \citep{li2024graph} enhance this process through attention mechanisms, ensuring quality edge predictions between synthetic and real nodes. Moreover, Mixup strategies, as seen in GraphMixup \citep{wu2022graphmixup} and GraphENS \citep{park2022graphens}, introduce sophisticated blending techniques to combat neighbor memorization issues and prevent out-of-domain sample generation.

\textit{Algorithm-level methodologies} focus on refining learning algorithms to accommodate class imbalance effectively. Model refinement strategies adapt graph representation learning architectures to improve performance by incorporating class-specific adjustments and modules tailored to handle imbalance. For instance, approaches like ACS-GNN \citep{ma2022attention} and EGCN \citep{wang2022effective} modify aggregation operations within graph neural networks to prioritize minority class information, ultimately enhancing classification accuracy. Additionally, loss function engineering tries to design customized loss functions to prioritize minority class errors or encourage wider class separations, a challenging task given the connectivity properties of nodes in graphs. Recent innovations like ReNode \citep{chen2021topology} and TAM \citep{song2022tam} integrate graph topology information into loss function designs, showcasing advancements in addressing class imbalance within the graph context. \citet{jervakani2024klce} proposed a KL regularization for imbalance node classification.

\textbf{Generalization Error and Node Classification:} There are two different scenarios in node classification, including inductive and transductive scenarios. Inductive learning node classification involves (semi-)supervised learning on graph data samples where the test data are unseen during training. In Inductive node classification For node classification tasks, \citep{verma2019stability} discussed the generalization error under node classification for GNNs via algorithm stability analysis in an inductive learning setting based on the SGD optimization algorithm. The work by \citet{zhou2021generalization} extended the results in \citep{verma2019stability} and found that increasing the depth of Graph Convolutional Neural Network (GCN) enlarges its generalization error for the node classification task. In Transductive Node Classification, based on transductive Rademacher complexity, a high-probability upper bound on the generalization error of the transductive node classification task was proposed by \citet{esser2021learning}. The transductive modeling of node classification was studied in \citep{oono2020optimization} and \citep{tang2023towards}. \citep{cong2021provable} presented an upper bound on the generalization error of GCNs for node classification via transductive uniform stability, building on the work of \citep{el2006stable}. In contrast, our research focuses on the task of imbalance transductive node classification and we provide an upper bound on population risk.

\textbf{Uncertainty in Self-training: }
The integration of uncertainty estimation in GNNs has been extensively explored, particularly in self-training frameworks \cite{wang2021be,yang2021self,zhao2021entropy}. Early contributions, such as \cite{li2018deeper}, introduced a simple strategy that selects the top K most confident nodes based on softmax probabilities for pseudo-labeling, expanding the labeled set iteratively. This approach was later refined by \cite{sun2020multi}, which employed a multi-stage self-training method, iteratively updating node labels to enhance GNN performance in sparse label settings. Extending these ideas, \cite{yang2021self} proposed a self-enhanced GNN framework that utilizes an ensemble of models to ensure consistency in pseudo-label predictions, thereby improving robustness against label noise. A common technique for uncertainty estimation in these methods is the use of entropy, as demonstrated in \cite{cai2017active, zhang2021alg}, which measures prediction confidence to guide node selection. Building on this, \cite{zhao2021entropy} introduced an entropy-aware self-training framework, incorporating an entropy-aggregation layer to account for graph structural information in uncertainty estimation. While traditional approaches in entropy-based uncertainty estimation, prioritize high-confidence nodes for pseudo-labeling, this practice might unintentionally reinforce distributional biases, as highlighted by \cite{liu2022confidence}, where over-reliance on highly certain nodes risks a distribution shift between the original labeled set and the expanded pseudo-labeled set. Furthermore, \cite{li2023informative} observed that nodes with very high confidence often contribute redundant information, limiting the diversity of the labeled dataset and potentially biased learning process. To address this concerning issue, in this work, we suggest to balance node selection across a broader confidence range to improve the performance of our model.
\section{Preliminaries}

\paragraph{Notations:}We adopt the following convention for random variables and their distributions in the sequel. 
A random variable is denoted by an upper-case letter (e.g., $Z$), its space of possible values is denoted with the corresponding calligraphic letter (e.g., $\mathcal{Z}$), and an arbitrary value of this variable is denoted with the lower-case letter (e.g., $z$). We denote the set of integers from 1 to $N$ by $[N]\triangleq \{1,\dots,N\}$; the set of 
measures over a space $\mathcal{X}$ with finite variance
is denoted
$\mathcal{P}(\mathcal{X})$. For a matrix $\mbX\in\bbR^{k\times q}$, we denote the $i-$th row of the matrix by $\mbX[i,:]$. The Euclidean norm of a vector
$X\in\bbR^q$ 
is $\norm{X}_2:=(\sum_{j=1}^q x_j^2)^{1/2}$. For a matrix $\mathbf{Y}\in\bbR^{k\times q}$, we let
$\norm{\mathbf{Y}}_{\infty}:=\max_{1\leq j\leq k}\sum_{i=1}^q|\mathbf{Y}[j,i]|$ and $\norm{\mathbf{Y}}_{F}:=\sqrt{\sum_{j=1}^k\sum_{i=1}^q\mathbf{Y}^2[j,i]}$.

\paragraph{Information Measures:} The KL divergence $\KLr(P\|Q)$ is given by
$\KLr(P\|Q):=\int_{\mathcal{Z}}\log\bigl(\frac{dP}{dQ}\bigr) dP$. We also define the cross-entropy between $P$ and $Q$, as  $H(P,Q)=-\int_{\mathcal{Z}}\log\bigl(dP\bigr) dQ$.

\paragraph{Graph data samples:} We consider transductive node classification for undirected graphs with $N$ nodes that have no self-loops or multiple edges.
Inputs to GNNs are node samples, 
which are comprised of their features and 
graph adjacency matrices. We denote the space of node feature for a node in the graph by $\mathcal{X}$. We consider the adjacency matrix $\mbA\in\mathcal{A}\subset \{0,1\}^{N\times N}$ to be fixed  with maximum node degree for $i$-th class
$\degmax_i$ and minimum node degree for $i$-th class
$\degmin_i$. The input pair of $i$-th node feature with dimension $k$ is denoted by $\mathbf{x}_i$. The GNN output (label) is denoted by $y\in \mathcal{Y}$ where $\mathcal{Y}$ is the space of labels and $|\mathcal{Y}|=k$.  Define $\mathcal{Z}=\mathcal{X}\times \mathcal{Y}$. Let $\mbZ_m=\{Z_i\}_{i=1}^m\in\mathcal{Z}$ denote the training of labeled dataset, where the $i-$th node sample is $Z_i=(X_i,Y_i)$. For transductive learning in node classification, we denote $\mbZ_u=\{Z_i\}_{i=1+m}^{m+u}\in\mathcal{Z}$ as the set of unlabeled dataset. We assume that the feature of nodes are i.i.d samples from the the feature space $\mathcal{X}$. We also denote the matrix of node features in the graph with $\mathbf{X}\in\mathbb{R}^{N\times k}$. We denote the learned distribution over classes as $P(\hat{\mathbf{Y}}|\mbX):=\{P(\hat{Y}=j|\mathbf{X})\}_{j=1}^k$ where $P(\hat{Y}=j|\mathbf{X}):=\frac{1}{n}\sum_{i=1}^{n} P(\hat{Y}=j|X_i)$ is the prediction of model for $j$-th class for given node-features $\{X_i\}_{i=1}^n$. Furthermore, we define $\mathbf{P}_k:=\{P_j\}_{j=1}^k$ as the target class distribution where $P_j = \frac{|m_j|}{\sum_{i=1}^k |m_i|}$ shows the probability of $j$-th class within the dataset.

\textbf{Graph filters:} Graph filters are linear functions of the adjacency matrix, $\mbA$, defined by $\mathrm{G}_f:\mathcal{A}\mapsto \bbR^{N \times N}$ where $N$ is the size of the graph, see in~\citep{defferrard2016convolutional}. Graph filters model the aggregation of node features in a GNN. For example, the symmetric normalized graph filter proposed by \citet{kipfsemi} is  $\mathrm{G}_f(\mbA)=\tilde{L}:=\tilde{D}^{-1/2}\Tilde{\mbA}\tilde{D}^{-1/2}$ where $\Tilde{\mbA}=I+\mbA$ is the augmented adjacency matrix, $\tilde{D}$ is the degree-diagonal matrix of $\Tilde{\mbA}$, and $I$ is the identity matrix. Furthermore, $\mathrm{deg}_i^{\max}$ and $\mathrm{deg}_i^{\min}$ are the maximum and minimum degree of nodes among $i$-th class of the graph with the augmented adjacency matrix, $\Tilde{\mbA}$. Another normalized filter, a.k.a. random walk graph filter~\citep{xupowerful2019}, is $\mathrm{G}_f(\mbA)=D^{-1}\mbA+I$, where $D$ is the degree-diagonal matrix of $\mbA$. The mean-aggregator is also a well-known aggregator defined as $\mathrm{G}_f(\mbA)=\tilde{D}^{-1}(\mbA+I)$. The sum-aggregator graph filter is defined by $\mathrm{G}_f(\mbA)=(\mbA+I)$.

\textbf{Loss function:} With $\mathcal{Y}$ the label space, the loss function $\ell: \mathbb{R} \times \mathcal{Y} \to \bbR$ is denoted as $\ell(\widehat{y},y)$, where $\widehat{y}$ is defined as the output of model, e.g., $\ell(h_\theta(\mathrm{G}_f(\mbA)[i,:]\mbX),y_i)$ is the loss function for $i$-th node in the graph where $h_\theta:\mathcal{X}\mapsto\mathcal{Y}$ is the hypothesis of our parameterized model which belongs to $\mathcal{H}$ with $\theta\in\Theta$.

\textbf{True and empirical risks:} The empirical risk function for labeled node samples $\mbX_m$ is given by 
\begin{equation} \label{eq: emp risk}
\begin{split}
     &\mrR(\mbZ_m,h_\theta):= \frac{1}{m}\sum_{i=1}^m \ell(h_\theta(\mathrm{G}_f(\mbA)[i,:]\mbX_m),y_i)\,.
     \end{split}
\end{equation}
We also define the true risk as the average of loss function over unlabeled nodes, $\mbX_u$, in transductive learning,
\begin{equation} \label{eq: true risk}
\begin{split}
     &\mrR(\mbZ_u,h_\theta):= \frac{1}{u}\sum_{i=m+1}^{m+u} \ell(h_\theta(\mathrm{G}_f(\mbA)[i,:]\mbX_u),y_i)\,.
     \end{split}
\end{equation}
We would like to study the performance of the model trained with the labeled data set $\mbZ_m=\{(\mbX_i,Y_i)\}_{i=1}^m$ and evaluated against the unlabeled dataset $\mbZ_u=\{(X_i,Y_i)\}_{i=m+1}^{m+u}$ loss, using transductive Rademacher complexity analysis~\citep{el2009transductive}.
\section{The Generalization Error of Imbalance Transductive Node Classification}\label{sec: gen analysis}
This section provides a generalization error upper bound under an imbalance scenario for transductive node classification via Rademacher complexity analysis. All proof details are deferred to Appendix~\ref{app: sec: gen analysis}.

For theoretical analysis, we need the following assumptions.
\begin{assumption}[Bounded node features]\label{Ass: Feature node bounded}
For every node, the features are contained in an $\ell_2$-ball of radius $B_f$. In particular, $\|\mbX [i;]\|_2\leq B_f$ for all nodes.
\end{assumption}
\begin{assumption}[Bounded norm of parameters]\label{ass: bounded param}
 There exists $U_F(i)\in\mbR$ such that $\norm{\theta_i}_{F}\leq U_F(i) $   where $\theta_i$ is parameter matrix of $i$-th layer.
\end{assumption}
\begin{assumption}[Activation functions]\label{ass: bounded activation function} The activation function $\varphi:\mathbb{R}\mapsto\mathbb{R}$ is $1$-Lipschitz \footnote{For the Tanh and Relu activation function satisfy this assumption.}, so that $|\varphi(x_1)-\varphi(x_2)|\leq  |x_1-x_2|$ for all $x_1,x_2\in \mbR$,  and zero-centered, i.e., $\varphi(0)=0$.  
\end{assumption}

In theoretical analysis, we consider binary classification, based on \citep{bartlett2017spectrally}, where $\gamma$-margin loss function for binary classification is proposed. For a positive $\gamma\in\mbR^{+}$,we have, $\ell_\gamma(y_1,y_2)=0$ and $\ell_\gamma(y_1,y_2)=\min(1,1-\frac{y_1y_2}{\gamma})$ for $y_1y_2\geq \gamma$ and $y_1y_2<\gamma$, respectively.

For transductive learning analysis, inspired by \citet{el2009transductive}, we define empirical transductive Rademacher complexity as for $p\in[0,1/2]$ and let the set $\pmb{\epsilon}:=\{\epsilon_i\}_{i=1}^{m+u}$ be i.i.d. with probability $P(\epsilon_i=1)=P(\epsilon_i=-1)=p$ and $P(\epsilon_i=0)=1-2p$ for all $i\in[m+u]$ and assuming $p=\frac{m+u}{(m+u)^2}$, \[\mathfrak{R}_{m+u}(\mathcal{H}):= (\frac{1}{m}+\frac{1}{u})\mathbb{E}_{\pmb{\epsilon}}\Big[\sum_{i=1}^{m+u}\epsilon_i h_\theta(x_i)\Big].\]
Similarly, we can define the empirical transductive Rademacher complexity for $i$-th class with $\mathfrak{R}_{m_i+u_i}(\mathcal{H})$.
Furthermore, in transductive learning, we consider $m_i$ and $u_i$ as labeled nodes and unlabeled nodes belong to the $i$-th class. We first provide an upper bound on population risk for transductive learning under the imbalance scenario.

\begin{proposition}\label{prop: general bound}
    Let $Q_i=(\frac{1}{u_i}+\frac{1}{m_i})$, $S_i=\frac{m_i+u_i}{(m_i+u_i-1/2)(1-1/(2\max(m_i,u_i)))}$  for $i=1,2$. Then, with probability at least $(1-\delta)$ over the choice of the training set from nodes of graphs, for all $h_\theta\in\mathcal{H}$,
   \begin{equation*}
    \begin{split}
         &R(\mbZ_u,h_\theta)\leq \sum_{i=1}^2 R_{\gamma}(\mbZ_{m_i},h_{\theta})+\frac{u_i}{u}\Big(\frac{1}{\gamma}\mathfrak{R}_{m_i+u_i}(\mathcal{H})
       \\&\qquad +c_0 Q_i\sqrt{\min(m_i,u_i)}+\sqrt{\frac{S_iQ_i}{2}\log(1/\delta)}\Big),
    \end{split}
    \end{equation*}
    where $\mathfrak{R}_{m_i+u_i}(\mathcal{H})$ is the transductive Rademacher complexity for $i$-th class, $R_{\gamma}(\mbZ_{m_i},h_{\theta})$ is the empirical risk based on $\gamma$-margin loss for $i$-th class  and $c_0=\sqrt{\frac{32\ln(4e)}{3}}$.
\end{proposition}
Next, we provide an upper bound on the transductive Rademacher complexity for general deep GNN.
\begin{proposition}\label{prop: TR RC bound}
    Given Assumptions~\ref{Ass: Feature node bounded}, \ref{ass: bounded param} and \ref{ass: bounded activation function}, then the following upper bound holds on the transductive Rademacher complexity of node classification for GNN with depth $d$ and graph filter $\mathrm{G}_f(\mbA)$,
    \begin{equation*}
        \mathfrak{R}_{m_i+u_i}(\mathcal{H})\leq \frac{C_1\|\mathrm{G}_f(\mbA)[i]\|_{\infty}^{2(d-1)}\Pi_{j=1}^{d} U_F(j) }{\sqrt{m_i+u_i}},
    \end{equation*}
    where $C_1=B_f (\sqrt{2\log(2) d}+1)$ and $\|\mathrm{G}_f(\mbA)[i]\|_{\infty}$ is the infinite norm of graph filter among $i$-th class.
\end{proposition}
Combining Proposition~\ref{prop: TR RC bound} with Proposition~\ref{prop: general bound}, we can derive the following upper bound on population risk.
\begin{theorem}\label{thm: main result}
     Under the same Assumptions in Proposition~\ref{prop: general bound} and Proposition~\ref{prop: TR RC bound}, the following upper bound holds on population risk for imbalance transductive node classification under a GCN model,
     \begin{equation*}
    \begin{split}
         &R(\mbZ_u,h_\theta)\leq \sum_{i=1}^2 R_{\gamma}(\mbZ_{m_i},h_{\theta})\\&\quad+\frac{u_i}{u}\Big(\frac{B_f (\sqrt{2\log(2) d}+1)\Pi_{j=1}^{d} U_F(j) }{\gamma\sqrt{(m_i+u_i)}}\times \Bigg(\frac{\degmax_i+1}{\deg^{\min}+1}\Bigg)^{d-1}
        \\&\quad +c_0 Q_i\sqrt{\min(m_i,u_i)}+\sqrt{\frac{S_iQ_i}{2}\log(1/\delta)}\Big),
    \end{split}
    \end{equation*}
where $U_F(i)$ is the maximum of Frobenius norm of $i$-th layer parameter. 
\end{theorem}
\textbf{Convergence rate discussion:} For large values of $u_i$ and $m_i$ $(i\in{1,2})$, the value of $S_i$ approaches one. Moreover, we find $Q_i = O(1/\min(m_i,u_i))$ and $\mathfrak{R}_{m_i+u_i}(\mathcal{H})=O(1/(\sqrt{m_i+u_i}))$. Consequently, assuming $u_i\ll m_i$ for both $i=1,2$, yields an overall convergence rate of $O(1/u)$. Conversely, when $m_i \ll u_i$, the convergence rate becomes $O(\max(\frac{1}{\sqrt{m_1}},\frac{1}{\sqrt{m_2}})$. It's noteworthy that in the balanced case, where $u_1=u_2$ and $m_1=m_2$, assuming $m_i \ll u_i$ for $i=1,2$ recover the convergence rate applicable to balanced node classification, dependent on the maximum degree per class and the minimum degree.

\textbf{Maximum and Minimum Degrees:}  In balance transductive node classification, e.g., \citep{oono2019graph,cong2021provable,tang2023towards}, the upper bound on population risk is dependent on the maximum and minimum degree of all nodes. Our bound in Theorem~\ref{thm: main result} reveals that for an imbalance scenario, the upper bound depends on the maximum degree of nodes per class and the minimum degree of all graph nodes. Note that, this result can also be applied to balance scenario by assuming $m_i=m$ for all classes.

\textbf{Discussion:} Theorem \ref{thm: main result} reveals that the population risk upper bound is fundamentally dominated by the number of samples in the minority class. This theoretical insight suggests a straightforward yet powerful approach: increasing the representation of minority classes can improve classification performance. While this observation emerges directly from our theoretical analysis, its implications are significant as it applies broadly to any transductive imbalanced node classification task. In the following section, we propose an algorithm that leverages this insight by strategically assigning pseudo-labels to unlabeled nodes, thereby effectively augmenting the minority class representation and improving node-classification performance.

\section{UPL Algorithm}

\begin{figure*}[t]
    \centering
    \includegraphics[width=0.7\textwidth]{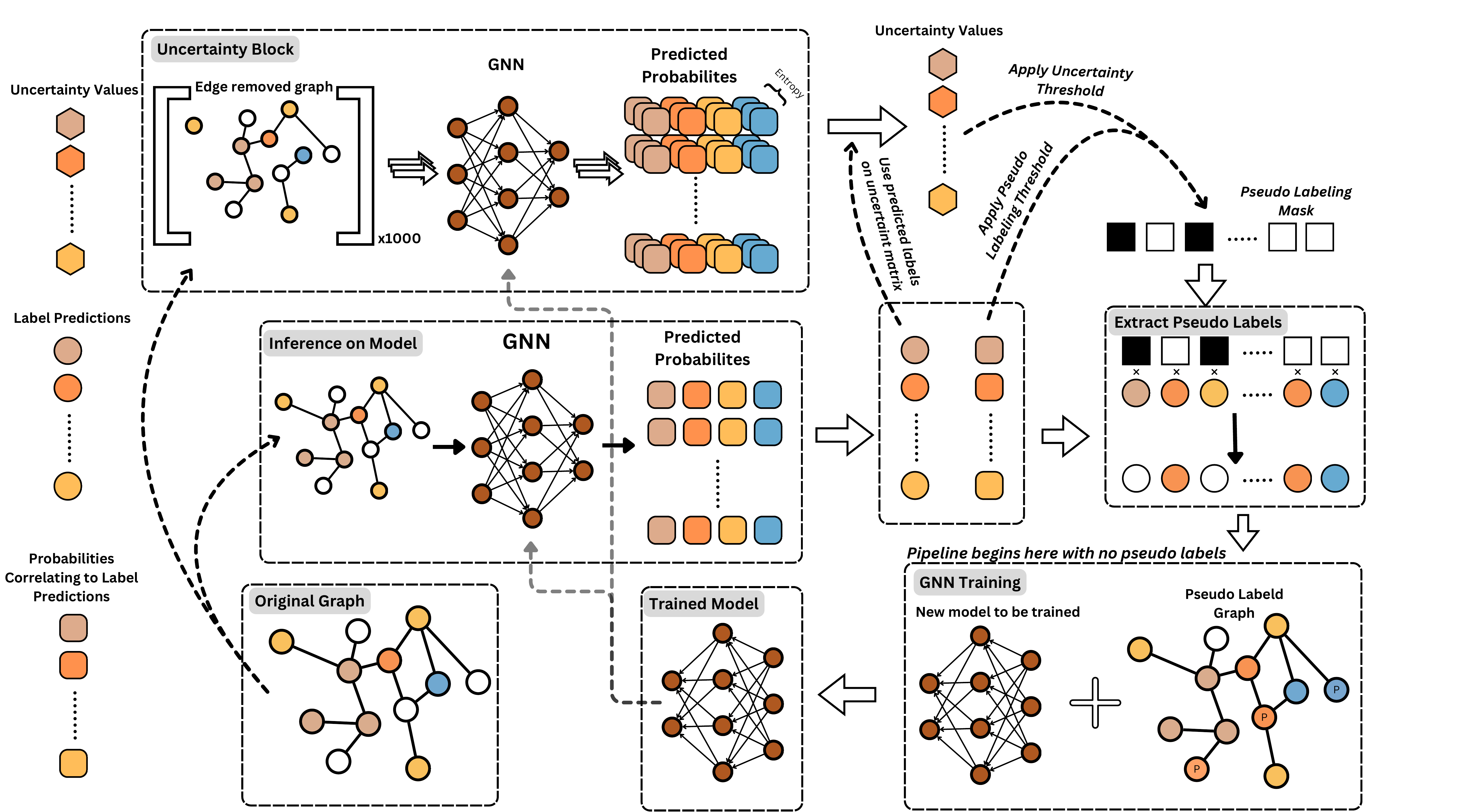}
    \caption{Pipeline of the UPL Algorithm}
    \label{fig:pipline}
\end{figure*}
In this section, we introduce the Uncertainty-aware Pseudo-labeling  (UPL) Algorithm to mitigate the imbalance effect in transductive node classification.

In the UPL algorithm, we add pseudo-labels to unlabeled nodes of minority classes to mitigate the effect of imbalance node classification. Our UPL algorithm contains two blocks: uncertainty-aware pseudo-labeling, and balanced-softmax loss function. The pipeline of the UPL algorithm is shown in Figure~\ref{fig:pipline}. Our UPL algorithm involves an iterative training approach, leveraging the pseudo-labeling technique to enhance the model's performance on graph-based data.

\subsection{Uncertainty-aware Pseudo-labeling}

Pseudo-labeling is a semi-supervised learning technique that enhances model performance by generating synthetic labels for unlabeled data using a previously trained model. This method effectively expands the training dataset, allowing the model to learn from a broader range of examples and improve accuracy. Different strategies for Pseudo-labeling are proposed. In this work inspired by the technique in \citep{rizve2021in}, we propose an iterative uncertainty-aware pseudo-labeling method where the uncertainty is computed based on the graph's topology. 

\textbf{Pseudo-labeling:} To refine the selection of pseudo-labels, we incorporate two thresholds: a lower threshold ($\eta_l$) and an upper threshold ($\eta_u$). As proposed by \citet{rizve2021in}, the lower threshold selects high-confidence samples. To avoid overfitting to these high-confidence samples, we introduce an upper threshold, $\eta_u$, which selects samples with intermediate confidence levels. This can be seen as a band-pass filter on the nodes, including those within a specific probability range. This method balances the use of confident predictions with the exploration of less certain ones, preventing the model from becoming biased towards the nodes' most confident nodes. To mitigate, the noise of pseudo-lableling process, we also add pseudo-labels to minority classes and avoid the pseudo-labels for majority classes during training.

\textbf{Uncertainty-Aware:} Our proposed Selective Edge Removal (SER) method addresses the challenge of poor calibration in GNN pseudo-labeling by introducing a topology-aware uncertainty estimation approach. The method works by strategically removing edges based on node degrees, where highly-connected nodes have a higher probability of edge removal, then making multiple inferences on these perturbed graph versions. By measuring the entropy of predictions across different perturbed versions and computing the variance of these entropy values, we can effectively estimate prediction uncertainty for each unlabeled node. Nodes showing consistent predictions (low entropy variance) across perturbations are considered more reliable and selected for pseudo-labeling, specifically retaining the lowest 90\% based on uncertainty values. This approach naturally accounts for both aleatoric and epistemic uncertainty while preserving the graph's core structure, offering a computationally efficient way to improve pseudo-label quality in GNN applications.

Finally, combining the thresholds for the pseudo-labeling process and the uncertainty estimation, we can select the pseudo-label for an unlabeled node, i.e.,
\begin{equation}
\begin{split}
     \tilde{Y}_i&=\mathds{1}(P(\hat{Y}_i|X_i) \geq \eta_l)\mathds{1}(P(\hat{Y}_i|X_i)\leq \eta_u)\\ &\mathds{1}\big(\mathrm{U}(P(\hat{Y}_i|X_i)) \leq Q_{\alpha}\big) ,
\end{split}
\end{equation}
where $P(\hat{Y}_i|X_i)$ is the prediction of our model for node sample $X_i$ over classes and $\mathrm{U}(\cdot)$ is variance of entropy of each nodes predictions among different perturbations,$\mathrm{U}(P(\hat{Y}_i|X_i)) = \mathbb{VAR}_k(H_k(i))$
where $H_k(i)$ is the entropy of predictions for $i$-th node in the $k$-th perturbed graph. Also $Q_{\alpha}$ represent the quantile function for some $\alpha\in(0,1)$ as follows:
\begin{equation}
\begin{split}
Q_{\alpha} = \text{quantile}\big({\mathrm{U}(P(\hat{Y}_j|X_j)) : j = 1, \dots, n}, \alpha\big).
\end{split}
\end{equation}

\subsection{Balanced Softmax loss function}
The balanced softmax loss function~\citep{ren2020balanced} is a variant of the traditional softmax loss designed to address the issue of class imbalance in multi-class classification problems. The primary idea is to introduce a class-specific weighting factor, $\beta_c$, which adjusts the contribution of each class to the overall loss. This weighting factor is inversely proportional to the class frequency, assigning higher weights to underrepresented classes and lower weights to overrepresented classes. The formulation of the balanced softmax loss is given by
\begin{equation}
 -\frac{1}{N} \sum_{i=1}^{N} \sum_{c=1}^{k} \beta_c y_{ic} \log(p_{ic}),
\end{equation}
where $N$ is the number of training samples, $k$ is the number of classes, $y_{ic}$ is the ground truth label (0 or 1) for sample $i$ and class $c$, $p_{ic}$ is the predicted probability for sample $i$ and $c$-th class , and $\beta_c$ is the weighting factor calculated as $\beta_c = \frac{(1 - \alpha)}{(1 - \alpha^{m_c})} \alpha^{(1 - m_c)}$. Here, $\alpha$ is a hyperparameter controlling the degree of balancing, and $m_c$ is the number of training samples belonging to $c$-th class. This adjustment ensures that the model pays more attention to the minority classes during training, improving performance on imbalanced datasets. Finally, the details and operational steps of the UPL algorithm are outlined in Appendix~\ref{app: UPL alg detail}.
 \vspace{-1em}
\section{Experiment}\label{sec: experiments}
In this section, we provide the results of experiments for our UPL algorithm. All experiments setup details are provided in Appendix~\ref{app: Exp details}. All experiments are conducted on GCN models. More GNN architectures, e.g., GAT and GraphSage, are studied in Appendix~\ref{App: more exp}.
\paragraph{Baselines:} In our experiments, we compare our proposed method against a blend of both foundational and state-of-the-art (SOTA) methods. In particular, we include
foundational baselines: Vanila, Re-weight \citep{japkowicz2002class}, DR-GCN~\citep{shi2020multi} and Balanced Softmax~\citep{ren2020balanced}. Furthermore, we also consider recent SOTA methods, including, PC Softmax~\citep{hong2021disentangling}, GraphSMOTE~\citep{zhao2021graphsmote}, GraphENS~\citep{park2022graphens}, TAM~\citep{song2022tam}, and Unreal \citep{yan2023unreal}.

\paragraph{Datasets:} In our experiments, we utilize citation network datasets—CiteSeer, Cora, and Pubmed \citep{sen2008collective}—which we categorize as homophilic graphs, where nodes represent documents and edges represent citation links. For heterophilic graphs, we employ Chameleon and Squirrel \citep{rozemberczki2021multi}, where nodes represent Wikipedia pages and edges denote links between them, as well as Wisconsin, a graph composed of web pages crawled from the Internet by \citet{craven2000learning}. Furthermore, we use the Amazon Computers and Amazon Photos datasets \cite{shchur2019pitfallsgraphneuralnetwork} in Appendix \ref{App: more exp}.

 \textbf{Imbalance Learning:} In a classification scenario with training dataset $z = \{z_1, \ldots, z_n\}$ where $z_i=(x_i,y_i)$, the imbalance occurs when the distribution $\mathbf{P}_k$ over $k$ classes within the dataset has larger probability mass points over some classes. The imbalance severity is often quantified using an imbalance ratio $\rho$, which measures the ratio between the number of samples in the most populated class (majority class) and the least populated class (minority class) and can be defined as follows,
\begin{equation}
     \quad \rho = \frac{\max_{j\in[k]} |m_j|}{\min_{j\in[k]} |m_j|}.
\end{equation} 
\vspace{-1em}

\textbf{Hyper-parameter Sensitivity}: We conducted a sensitivity analysis on the two thresholds used in pseudo-labeling to evaluate the model's performance at various values of $\eta_l$ and $\eta_u$. The results of this parameter sweep are presented in Figure \ref{fig:sweep_bounds}.

The chart on the left shows the F1-score as $\eta_l$ is varied within the range $(0.25, 1.0)$, while $\eta_u$ is set to $\min(\eta_l + 0.3, 1.0)$. Based on this lower-threshold sweep, we observe that the optimal value of $\eta_l$ for this case lies within the range $(0.25, 0.45)$. A similar trend is observed across other datasets.

The right plot in Figure \ref{fig:sweep_bounds} depicts the results of varying $\eta_u$ within the range $(0.3, 1.0)$, with $\eta_l$ fixed at 0.3. These results further support the importance of using an upper threshold. Specifically, setting $\eta_u$ within the range $(0.35, 0.45)$ improves the overall F1-score on both the validation and test sets.

\begin{figure}[ h!]
    \centering
    \includegraphics[width=1.0\columnwidth]{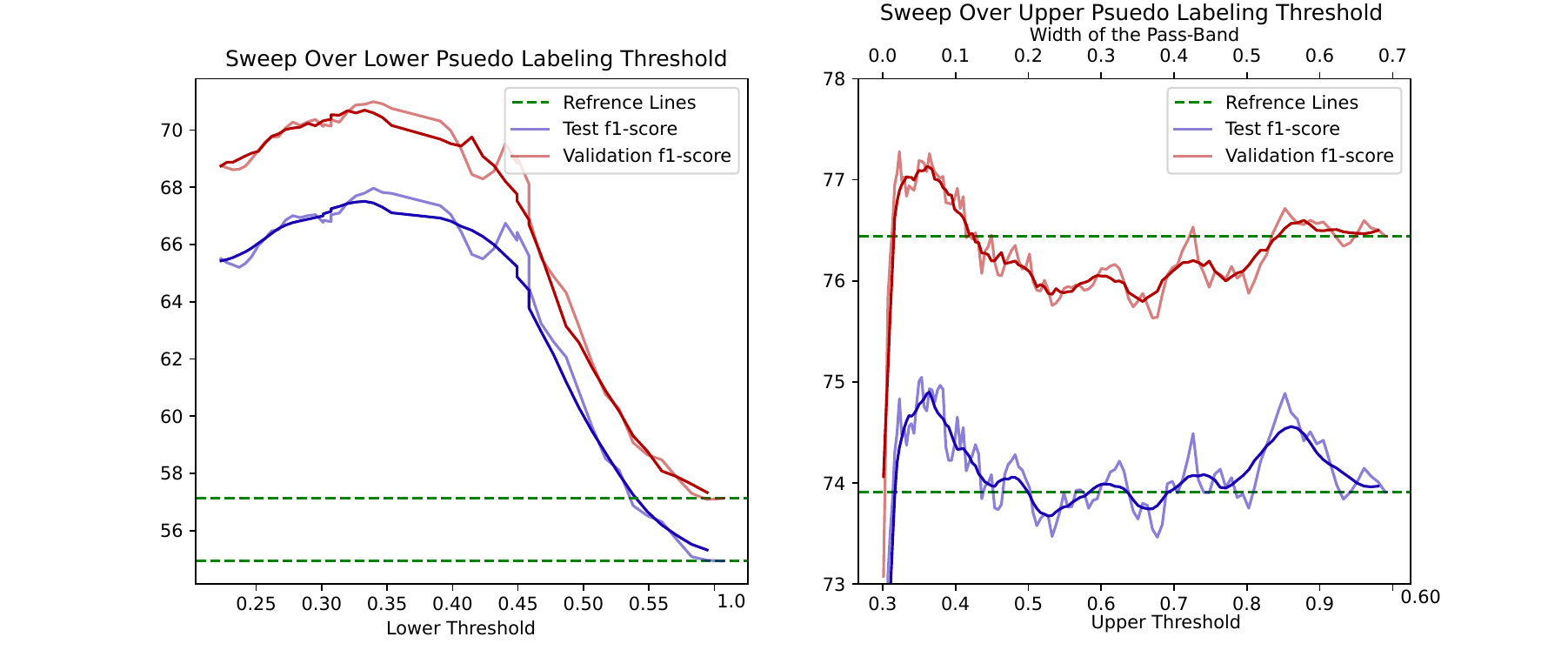}
    \caption{F1-score for different values of upper and lower thresholds for pseudo-labeling. The Figure on the left indicates the sweeping over $\eta_l$ while $\eta_u=0.3+\eta_l$ for the CiteSeer dataset. The Figure on the right represents sweeping over $\eta_u$ from $0.3$ to $1.0$ for Cora. In both Figures, the red line represents the validation results, and blue shows the result for the test set.}
    \label{fig:sweep_bounds}
\end{figure}


\subsection{Experiment results}
To further study the effect of UPL, we have conducted experiments on multiple datasets, a fraction of which has been shown in Table \ref{tb:main_chart_homo} and \ref{tb:main_chart_hetero}.
We report averaged balanced accuracy (bAcc)\footnote{Balanced accuracy is calculated as $\frac{1}{2}\big(\frac{\mathrm{TP}}{\mathrm{TP}+\mathrm{FN}}+\frac{\mathrm{TN}}{\mathrm{TN}+\mathrm{FP}} \big)$ where $\mathrm{TP}$ is true positive, $\mathrm{FN}$ is false negative and $\mathrm{TN}$ is true negative.} and F1-score with the standard errors for 10 repetitions on the GCN model. In our experiments, datasets are generated with an imbalance ratio of $\rho=10$. We explore other imbalance ratios in Appendix~\ref{App: more exp}.

\begin{table*}[ht!]
\caption{\small Experimental results of our algorithm (UPL) compared to other baselines. All the experiments have been done using the GCN architecture. We report the averaged balanced accuracy, averaged F1-score and the standard error of each experiment for Cora, CiteSeer and PubMed datasets. All the results have been calculated over 10 repetitions and 10 pseudo-labeling iterations.}
\begin{center}
\begin{scriptsize}
 \resizebox{0.9\textwidth}{!}{
\begin{tabular}{@{}rlcc|cc|cc@{}}
\toprule
 & \multirow{1}{*}{\textbf{Dataset}} & \multicolumn{2}{c}{Cora} & \multicolumn{2}{c}{CiteSeer} & \multicolumn{2}{c}{PubMed}  \\ 
\cline{1-8}\rule{0pt}{2.2ex}
& \textbf{\makecell{Imbalance Ratio \\($\rho=10$)}} & bAcc. & F1 & bAcc. & F1 & bAcc. & F1 \\
\cline{1-8}
\multirow{11}{*}
                     & Vanilla 
                     & 53.58 \tiny{$\pm 0.70$} & 48.52 \tiny{$\pm 1.09$}
                     & 35.29 \tiny{$\pm 0.34$} & 22.62 \tiny{$\pm 0.39$}
                     & 62.22 \tiny{$\pm 0.35$} & 48.99 \tiny{$\pm 1.02$} \\

                     & Re-Weight\citep{japkowicz2002class}
                     & 65.52 \tiny{$\pm 0.84$}& 65.54 \tiny{$\pm 1.20$}
                     & 44.52 \tiny{$\pm 1.22$}& 38.85 \tiny{$\pm 1.62$}
                     & 70.17 \tiny{$\pm 1.25$}& 66.37 \tiny{$\pm 1.73$}
                     \\
                    
                     & PC Softmax \cite{hong2021disentangling}
                     & 67.79 \tiny{$\pm 0.92$}& 67.39 \tiny{$\pm 1.08$}
                     & 49.81 \tiny{$\pm 1.12$}& 45.55 \tiny{$\pm 1.26$}
                     & 70.20 \tiny{$\pm 0.60$}& 68.83 \tiny{$\pm 0.73$}
                     \\
                     
                     & DR-GCN \cite{shi2020multi}
                     & 60.17 \tiny{$\pm 0.83$}& 59.31 \tiny{$\pm 0.97$}
                     & 42.64 \tiny{$\pm 0.75$}& 38.22 \tiny{$\pm 1.22$}
                     & 65.51 \tiny{$\pm 0.81$}& 64.95 \tiny{$\pm 0.53$}
                     \\
                     & GraphSMOTE \cite{zhao2021graphsmote}
                     & 62.66 \tiny{$\pm 0.83$} & 61.76 \tiny{$\pm 0.96$}
                     & 34.26 \tiny{$\pm 0.89$} & 28.31 \tiny{$\pm 1.48$}
                     & 68.94 \tiny{$\pm 0.89$} & 64.17 \tiny{$\pm 1.43$}\\

                     & BalancedSoftmax \cite{ren2020balanced}
                     & 67.75 \tiny{$\pm 1.09$} & 66.47 \tiny{$\pm 1.27$}
                     & 52.09 \tiny{$\pm 1.61$} & 48.55 \tiny{$\pm 2.01$}
                     & 71.78 \tiny{$\pm 0.82$} & 71.18 \tiny{$\pm 0.94$} \\

                     & GraphENS \cite{park2022graphens}+ TAM \cite{song2022tam}
                     & 69.42 \tiny{$\pm 0.90$} & 68.70 \tiny{$\pm 1.08$}
                     & 53.10 \tiny{$\pm 1.82$} & 50.67 \tiny{$\pm 2.28$}
                     & 71.70 \tiny{$\pm 0.95$} & 69.01 \tiny{$\pm 1.72$} \\

                     & Unreal \cite{yan2023unreal}
                     & \textbf{78.33} \tiny{$\pm 1.04$} & \textbf{76.44} \tiny{$\pm 1.06$}
                     & \underline{65.63} \tiny{$\pm 1.38$} & \underline{64.94} \tiny{$\pm 1.38$}
                     & \underline{75.35} \tiny{$\pm 1.41$} & \underline{73.65} \tiny{$\pm 1.43$} \\
                     
                     \cline{1-8}

                     & \textbf{UPL}
                     & \underline{76.16} \tiny{$\pm 0.76$} & \underline{74.53} \tiny{$\pm 0.85$}
                    & \textbf{68.18} \tiny{$\pm 0.27$} & \textbf{67.78} \tiny{$\pm 0.25$}
                    & \textbf{77.50} \tiny{$\pm 0.44$} & \textbf{77.19} \tiny{$\pm 0.43$} \\

\bottomrule
\end{tabular}}

\end{scriptsize}
\end{center}
\label{tb:main_chart_homo}
\vspace{-0.05in}
\end{table*}

\begin{table*}[ht!]
\caption{\small Experimental results of our algorithm (UPL) compared to other baselines. All the experiments have been done using the GCN architecture. We report the averaged balanced accuracy, averaged F1-score and the standard error of each experiment for Cora, CiteSeer and PubMed datasets. All the results have been calculated over 10 repetitions and 10 pseudo-labeling iterations.}
\begin{center}
\begin{scriptsize}
\setlength{\columnsep}{1pt}%
 \resizebox{0.9\textwidth}{!}{
\begin{tabular}{@{}rlcc|cc|cc@{}}
\toprule
 &\textbf{Dataset} & \multicolumn{2}{c}{Chameleon} & \multicolumn{2}{c}{Squirrel} & \multicolumn{2}{c}{Wisconsin}  \\ 
\cline{1-8}\rule{0pt}{2.2ex}
& \textbf{\makecell{Imbalance Ratio\\ ($\rho=10$)}} & bAcc. & F1 & bAcc. & F1 & bAcc. & F1 \\
\cline{1-8}
\rule{0pt}{2.5ex}  
                     & Vanilla 
                     & 57.01 \tiny{$\pm 0.53$} & 55.61 \tiny{$\pm 0.67$}
                     & 38.07 \tiny{$\pm 0.57$} & 34.81 \tiny{$\pm 0.69$}
                     & 34.51 \tiny{$\pm 2.31$} & 31.45 \tiny{$\pm 2.66$} \\

                     & Re-Weight \cite{japkowicz2002class}
                     & 36.07 \tiny{$\pm 0.87$}& 35.61 \tiny{$\pm 0.81$}
                     & 26.92 \tiny{$\pm 0.53$}& 25.04 \tiny{$\pm 0.59$}
                     & 44.13 \tiny{$\pm 3.08$}& 40.74 \tiny{$\pm 3.27$}
                     \\
                    
                     & PC Softmax \cite{hong2021disentangling}
                     & 36.86 \tiny{$\pm 1.04$}& 36.24 \tiny{$\pm 1.01$}
                    
                     & 26.49 \tiny{$\pm 0.59$}& 25.73 \tiny{$\pm 0.49$}
                     
                     & 30.90 \tiny{$\pm 3.10$}& 28.15 \tiny{$\pm 2.16$}
                     \\
                     
                     & DR-GCN \cite{shi2020multi}
                     & 33.34 \tiny{$\pm 0.81$}& 29.60 \tiny{$\pm 0.79$}
                     
                     & 23.34 \tiny{$\pm 0.43$}& 18.20 \tiny{$\pm 0.49$}
                     
                     & 29.44 \tiny{$\pm 1.36$}& 27.08 \tiny{$\pm 1.37$}
                     \\
                     
                     & GraphSMOTE \cite{zhao2021graphsmote}
                     & 52.38 \tiny{$\pm 0.77$}& 50.03 \tiny{$\pm 0.76$}
                     & 41.30 \tiny{$\pm 0.47$}& 39.07 \tiny{$\pm 0.60$}
                     & 45.36 \tiny{$\pm 4.21$}& 40.91 \tiny{$\pm 4.39$}
                     \\

                     & BalancedSoftmax \cite{ren2020balanced}
                     & \underline{58.60} \tiny{$\pm 1.06$} & \underline{58.20} \tiny{$\pm 1.00$}
                     & 41.81 \tiny{$\pm 0.64$} & 41.27 \tiny{$\pm 0.71$}
                     & \underline{}{44.44} \tiny{$\pm 3.46$} & \underline{41.29} \tiny{$\pm 2.85$} \\
                     
                     & GraphENS \cite{park2022graphens}+ TAM \cite{song2022tam}
                     & 40.98 \tiny{$\pm 0.85$} & 38.93 \tiny{$\pm 0.84$}
                     & 27.98 \tiny{$\pm 0.90$} & 24.04 \tiny{$\pm 0.70$}
                     & 41.38 \tiny{$\pm 2.78$} & 34.56 \tiny{$\pm 2.22$} \\

                     & Unreal
                     & 58.18 \tiny{$\pm 0.94$} & 56.98 \tiny{$\pm 1.03$}
                     & \textbf{45.65} \tiny{$\pm 1.08$} & \textbf{44.78} \tiny{$\pm 1.07$}
                     & 41.26 \tiny{$\pm 5.19$} & 32.14 \tiny{$\pm 2.15$}\\

                     \cline{1-8}
                     & \textbf{UPL}
                     & \textbf{61.19} \tiny{$\pm 0.66$} & \textbf{60.71} \tiny{$\pm 0.71$}
                     & \underline{43.34} \tiny{$\pm 0.25$} & \underline{42.82} \tiny{$\pm 0.07$}
                     & \textbf{48.70} \tiny{$\pm 2.88$} & \textbf{42.72} \tiny{$\pm 2.38$} \\

\bottomrule
\end{tabular}}

\end{scriptsize}
\end{center}
\label{tb:main_chart_hetero}
\vspace{-0.05in}
\end{table*}

As shown in Tables \ref{tb:main_chart_homo} and \ref{tb:main_chart_hetero}, our UPL algorithm outperforms other baselines across both homophilic and heterophilic datasets in most cases. Notably, UPL demonstrates significantly lower performance variance than competing methods, indicating better reliability in its predictions. Table \ref{tb:main_chart_hetero} also highlights a common challenge which is a marked decline in performance on heterophilic datasets. In contrast, our method consistently delivers better results in most cases, even in these challenging scenarios. Moreover, based on theoretical results and the fact that the maximum degree of nodes per class in heterophilic datasets is large, then we expect lower accuracy over these datasets.

\begin{table*}[htb!]
\vspace{-1em}
\caption{\small Experimental results of Ablation study of our proposed algorithm. All the experiments have been done using the GCN architecture. We report the averaged balanced accuracy, averaged F1-score and the standard error of each experiment for all datasets.}
\begin{center}
\begin{scriptsize}
\setlength{\columnsep}{1pt}%
\centering
  \resizebox{\textwidth}{!}{
\begin{tabular}{@{}rlcc|cc|cc|cc|cc|cc@{}}
\toprule
 & \textbf{Dataset} & \multicolumn{2}{c}{Cora} & \multicolumn{2}{c}{CiteSeer} & \multicolumn{2}{c}{PubMed} & \multicolumn{2}{c}{Chameleon} & \multicolumn{2}{c}{Squirrel} & \multicolumn{2}{c}{Wisconsin}  \\ 
\cline{1-14}\rule{0pt}{2.2ex}
& \textbf{\makecell{Imbalance Ratio\\ ($\rho=10$)}} & bAcc. & F1 & bAcc. & F1 & bAcc. & F1 & bAcc. & F1 & bAcc. & F1 & bAcc. & F1\\
\cline{1-14}
\rule{0pt}{2.5ex}  
                     & Vanilla 
& 53.28 \tiny{$\pm 0.65$} & 48.12 \tiny{$\pm 1.03$}
& 35.53 \tiny{$\pm 0.33$} & 22.67 \tiny{$\pm 0.41$}
& 61.96 \tiny{$\pm 0.33$} & 48.97 \tiny{$\pm 0.89$}
& 57.24 \tiny{$\pm 0.63$} & 55.82 \tiny{$\pm 0.75$}
& 37.70 \tiny{$\pm 0.58$} & 34.45 \tiny{$\pm 0.71$}
& 34.51 \tiny{$\pm 2.31$} & 31.45 \tiny{$\pm 2.66$}
\\
                     & + BS
& 67.51 \tiny{$\pm 1.27$} & 66.29 \tiny{$\pm 1.32$}
& 51.83 \tiny{$\pm 1.58$} & 48.56 \tiny{$\pm 1.87$}
& 72.22 \tiny{$\pm 0.83$} & 71.31 \tiny{$\pm 0.99$}
& 58.98 \tiny{$\pm 0.86$} & 58.44 \tiny{$\pm 0.87$}
& 41.76 \tiny{$\pm 0.49$} & 40.91 \tiny{$\pm 0.50$}
& 44.44 \tiny{$\pm 3.46$} & 41.29 \tiny{$\pm 2.85$}
\\
                     & + PL
& 68.86 \tiny{$\pm 1.51$} & 66.45 \tiny{$\pm 1.86$}
& 48.81 \tiny{$\pm 2.62$} & 41.75 \tiny{$\pm 3.74$}
& 73.19 \tiny{$\pm 1.42$} & 71.98 \tiny{$\pm 2.08$}
& 60.24 \tiny{$\pm 0.90$} & 59.63 \tiny{$\pm 0.95$}
& 39.34 \tiny{$\pm 0.56$} & 37.15 \tiny{$\pm 0.58$}
& 38.62 \tiny{$\pm 2.28$} & 36.51 \tiny{$\pm 2.54$}
 \\
                     & + Uncertainty
& 53.28 \tiny{$\pm 0.65$} & 48.12 \tiny{$\pm 1.03$}
& 35.53 \tiny{$\pm 0.33$} & 22.67 \tiny{$\pm 0.41$}
& 61.96 \tiny{$\pm 0.33$} & 48.97 \tiny{$\pm 0.89$}
& 57.02 \tiny{$\pm 0.76$} & 55.78 \tiny{$\pm 0.85$}
& 37.78 \tiny{$\pm 0.57$} & 34.68 \tiny{$\pm 0.75$}
& 34.51 \tiny{$\pm 2.31$} & 31.45 \tiny{$\pm 2.66$}
\\
                     & + BS + PL
& \underline{75.27} \tiny{$\pm 0.92$} & \underline{73.59} \tiny{$\pm 1.02$}
& \underline{64.70} \tiny{$\pm 1.17$} & \underline{64.34} \tiny{$\pm 1.26$}
& \underline{76.08} \tiny{$\pm 0.28$} & \underline{76.08} \tiny{$\pm 0.29$}
& \textbf{61.26} \tiny{$\pm 0.73$} & \textbf{60.73} \tiny{$\pm 0.74$}
& \underline{41.83} \tiny{$\pm 0.42$} & \underline{41.04} \tiny{$\pm 0.48$}
& \textbf{47.72} \tiny{$\pm 4.08$} & \textbf{44.03} \tiny{$\pm 3.44$}
\\
                     & + Uncertainty + BS 
& 67.51 \tiny{$\pm 1.27$} & 66.29 \tiny{$\pm 1.32$}
& 51.83 \tiny{$\pm 1.58$} & 48.56 \tiny{$\pm 1.87$}
& 72.22 \tiny{$\pm 0.83$} & 71.31 \tiny{$\pm 0.99$}
& 59.04 \tiny{$\pm 0.85$} & 58.44 \tiny{$\pm 0.88$}
& 41.70 \tiny{$\pm 0.48$} & 40.88 \tiny{$\pm 0.54$}
& 44.44 \tiny{$\pm 3.46$} & 41.29 \tiny{$\pm 2.85$}
\\
                     & + Uncertainty + PL
& 67.51 \tiny{$\pm 2.16$} & 65.28 \tiny{$\pm 2.59$}
& 40.01 \tiny{$\pm 2.01$} & 28.25 \tiny{$\pm 2.31$}
& 74.85 \tiny{$\pm 0.75$} & 73.80 \tiny{$\pm 0.89$}
& 59.21 \tiny{$\pm 0.74$} & 58.14 \tiny{$\pm 0.77$}
& 38.39 \tiny{$\pm 0.67$} & 35.86 \tiny{$\pm 0.87$}
& 36.92 \tiny{$\pm 2.62$} & 34.74 \tiny{$\pm 2.76$}
\\
                    \cdashline{2-14}
                     & UPL
& \textbf{76.16} \tiny{$\pm 0.76$} & \textbf{74.53} \tiny{$\pm 0.85$}
& \textbf{68.18} \tiny{$\pm 0.27$} & \textbf{67.78} \tiny{$\pm 0.25$}
& \textbf{77.50} \tiny{$\pm 0.44$} & \textbf{77.19} \tiny{$\pm 0.43$}
& \underline{61.19} \tiny{$\pm 0.66$} & \underline{60.71} \tiny{$\pm 0.71$}
& \textbf{42.34} \tiny{$\pm 0.67$} & \textbf{41.76} \tiny{$\pm 0.68$}
& \underline{}{48.70} \tiny{$\pm 2.88$} & \underline{42.72} \tiny{$\pm 2.38$}
\\
\bottomrule
\end{tabular}}

\end{scriptsize}
\end{center}
\label{tb:Ablation_homo}
\end{table*}
\vspace{-0.05in}

\textbf{Uncertainty methods:} We evaluate our approach's uncertainty estimation on the Cora dataset, comparing against two uncertainty estimation methods for GNN as baselines: CaGCN~\cite{wang2021be} and ALG~\cite{zhang2021alg} in Table~\ref{tb:uncertainty_methods}. Our results show that imbalanced data significantly impacts uncertainty estimation performance in existing graph-based models.
\begin{table}[H]
\centering
\caption{Comparison of our model with CaGCN and Alg as uncertainty estimation in GNN utilizing uncertainty in graphs has been tested under the imbalance scenario with ratio $\rho=10$.}
\label{tb:uncertainty_methods}
\vspace{-0.5em}
\resizebox{0.6\columnwidth}{!}{\begin{tabular}{c|c|c}
\toprule
\textbf{Methods} & \textbf{Cora} & \textbf{CiteSeer} \\ \cline{1-3}
CaGCN 
& $46.23$ \tiny{$\pm 0.49$}
& $23.42$ \tiny{$\pm 0.74$}
\\ 
Alg   
& $67.48$ \tiny{$\pm 0.76$}
& $27.50$ \tiny{$\pm 0.85$}
\\ 
UPL   
& $74.51$ \tiny{$\pm 0.65$}
& $67.78$ \tiny{$\pm 0.25$}
\\
\bottomrule
\end{tabular}}
\end{table}
 
\textbf{Comparison with UNREAL:} Our algorithm, UPL, demonstrates notable computational advantages over the UNREAL algorithm \citep{yan2023unreal}. UPL's efficiency stems from two key factors: first, it employs a more selective pseudo-labeling process that focuses on minority classes, resulting in fewer pseudo-labeled nodes overall. Second, UPL eliminates the need for two computationally intensive components present in UNREAL—node reordering and clustering—which contribute significantly to that algorithm's complexity. Instead, UPL introduces a novel framework that incorporates uncertainty estimation and an upper threshold for pseudo-labeling. In contrast, the UNREAL algorithm operates through a series of three main steps, 
\begin{itemize}[nosep, leftmargin=*]
    \item Dual Pseudo-tag Alignment Mechanism (DPAM): This step involves clustering all nodes, which inherently have a complexity of $O(N^2)$ for $N$ nodes, and then inferring the GCN model.
    \item Node-Reordering: Utilizes the Rank-Biased Overlap (RBO) algorithm \citep{webber2010similarity} to merge two sorted lists generated by clustering and GCN inference. This merging process also presents a complexity of $O(N^2 )$ for $N$ nodes.
    \item Discarding Geometrically Imbalanced Nodes (DGIN): This step computes the second nearest cluster center for each node, leading to a complexity of $O(Nk)$, where $k$ is the predefined number of clusters,
\end{itemize}
where each iterated $T$ times, potentially leading to increased computational overhead. By reducing both the number of operations each step, UPL achieves improved efficiency without sacrificing performance. In particular, our UPL algorithm simplifies the computational process significantly,
\begin{itemize}[nosep, leftmargin=*]
    \item Uncertainty estimation: UPL involves sampling the edges to be removed based on their distribution to estimate the uncertainty. This process is iterated $S_k$ times.
    \item Subsequently, similar to UNREAL, we infer the GCN model.
\end{itemize}
Notably our complexity is less than the UNREAL, as we avoid DGIN and RBO processes in UNREAL. More discussion is provided in App.~\ref{app:comp_unreal}.
\vspace{-0.1in}
\subsection{Ablation study}
We evaluate key components' individual and combined effects on our UPL algorithm to discover their contributions in transductive node classification tasks under imbalance scenario, for both homophilic and heterophilic datasets, in Table~\ref{tb:Ablation_homo}, respectively. For this purpose, we use BalancedSoftmax \textbf{(BS)}, which adjusts softmax to favor minority classes; Uncertainty-aware Pseudo-labeling \textbf{(UPL)}, which uses generated pseudo-labeled samples to generate more training data; aiming to align predicted class distributions with actual distributions. In Tables~\ref{tb:Ablation_homo}, The Vanilla model serves as our baseline which is log-loss function.  \newline The results demonstrate the effectiveness of various strategies for addressing class imbalance in node classification tasks. Across all datasets, the UPL method consistently excels, showcasing the robustness of integrating approaches. Notably, combinations like BS + UPL frequently lead to significant performance improvements, highlighting the synergy between direct imbalance mitigation and the strategic use of pseudo-labeled data. 
\vspace{-1em}
\section{Conclusion and Future Works}
In this study, we make several key contributions to transductive node classification in imbalanced scenarios. First, we establish a theoretical upper bound on the population risk that depends on the maximum degree node within each class. Second, we develop an uncertainty-aware algorithm for identifying pseudo-labels for minority class nodes, introducing a novel selective edge removal method for uncertainty estimation. Our approach demonstrates superior performance compared to baseline methods across both homophilic and heterophilic datasets.

 While our theoretical analysis currently focuses on $\gamma$-margin loss functions in binary classification, several promising directions emerge for future work. From a theoretical perspective, we aim to extend our results to multi-class scenarios and broader loss functions, as well as expand the analysis to Inductive node classification settings where unlabeled nodes are not observed during training. Additionally, we plan to investigate theoretical bounds for heterophilic graphs to strengthen our framework's theoretical foundation.

As future work, we envision integrating our approach with H2GCN approach~\citep{zhu2020beyond} to enhance performance on heterophilic datasets. We also plan to incorporate advanced pseudo-labeling techniques such as Meta-Pseudo Label \citep{pham2021meta} or Flexmatch \citep{zhang2021flexmatch} to improve label quality. Finally, we intend to explore subgraph-based selective edge removal approaches for more sophisticated uncertainty estimation.

\clearpage
\section*{Impact Statement}
This paper presents work whose goal is to advance the field of 
Machine Learning. There are many potential societal consequences 
of our work, none which we feel must be specifically highlighted here.

\bibliographystyle{icml2025}
\bibliography{biblography}

\newpage
\appendix
\onecolumn

\addcontentsline{toc}{section}{Appendix} 
\part{Appendix} 
\parttoc 

\newpage
\section{Notations}\label{app:notation}
Notations in this paper are summarized in Table~\ref{Table: notation}.

\begin{table}[ht]
    \centering
    \caption{Summary of notations in the paper}
    \resizebox{\linewidth}{!}{
    \begin{tabular}{cl|cl}
        \toprule
        Notation & Definition & Notation & Definition \\
        \midrule
        $\mathcal{A}$ & Adjacency matrices space & $X_i$ & $i$-th node feature \\
        $\degmax_i$ & Maximum node degree i-th class & $\degmin_i$ & Minimum node degree i-th class \\
        $Y_i$ & $i$-th node's label & $Z_i$ & $i$-th node sample pair $(X_i, Y_i)$ \\
        $\mbZ_m$ & The set of data training samples & $\mbZ_u$ & The unlabeled samples \\
        $P(.)$ & Prediction of model & $m_i$ & Number of nodes in $i$-th class \\
        $k$ & Number of classes & $P_j$ & $\frac{|m_j|}{\sum_{i=1}^k{|m_i|}}$ \\
        $d_{\min}$ & Minimum node degree of among all classes & $\mathrm{G}_f(\mbA)$ & Graph filter with input matrix $\mbA$ \\
        $l(\hat y , y)$ & Loss function & $h_\theta$ & Parameterized model \\
        $B_f$ & $l_2$-ball of radius & $\theta_i$ & Parameter matrix of i-th layer \\
        $\varphi$ & Activation function &  $\mathfrak{R}_{m+u}(\mathcal{H})$ & Empirical Rademacher complexity   \\
       
          $R(Z_u, h_\theta)$ & True risk over unlabeled dataset  & $R(Z_m, h_\theta)$ & Empirical risk over labeled dataset\\
        $U_F(i)$ & Maximum of Frobenius norm i-th layer parameters & $U(.)$ & Uncertainty Function \\
        $\eta_l$ & Lower thresholds of Pseudo-labeling process & $\eta_u$ & Upper thresholds of Pseudo-labeling process \\
        $\KLr(P\|Q)$ & $\int_{\mathbb{Z}} \log(\mrd P/\mrd Q) \mrd P$ & $H(P,Q)$ & $\int_{\mathbb{Z}} \log(\mrd P) \mrd Q$
        \\$B_f$ & Bound on node features  & $\rho$ & Imbalance Ratio \\
        $\norm{\mathbf{Y}}_{F}$ & $\sqrt{\sum_{j=1}^k\sum_{i=1}^q\mathbf{Y}^2[j,i]}$ & $\norm{\mathbf{Y}}_{\infty}$ & $\max_{1\leq j\leq k}\sum_{i=1}^q|\mathbf{Y}[j,i]|$ \\
        $\norm{\mathbf{Y}}_{2}$ & $\left(\sum_{j=1}^q |q| x_j^2\right)^{1 / 2}$  & $Q_{\alpha}$ & Quantile of order $\alpha$\\
        $t$ & Number of itterations & $S_k$ & Number of edges removed in each itteration\\
        \bottomrule
    \end{tabular}}
    \label{Table: notation}
\end{table}

\section{Other Related Works}\label{app:other_related_work}
\textbf{Pseudo-labeling in semi-supervised learning:} Pseudo-labeling, introduced in~\citep{lee2013pseudo}, involves training a model using labeled data and assigning pseudo-labels to the unlabeled data based on the model's predictions. These pseudo-labels are then used to construct another model, which is trained in a supervised manner using both labeled and pseudo-labeled data. Under different scenarios, it is shown that the pseudo-labeling is effective,  \citep{arazo2020pseudolabeling} and \citep{rizve2021in}. \citep{kou2023how} shows that semi-supervised learning with pseudo-labeling can achieve near-zero test loss under some conditions. The study by \citet{DBLP:journals/corr/abs-2003-10580} introduced meta pseudo-labeling. This method enhanced pseudo-labels' accuracy by incorporating feedback from the student model. \citep{rizve2021in} proposed confidential-based pseudo-label generation for training a network with unlabeled data. \citep{arazo2020pseudolabeling} suggests soft-labeling with the MixUp method to reduce over-fitting to
model predictions and confirmation bias. In contrast, our work is focused on transductive node classification, and we employed the Pseudo-labels to mitigate the imbalance effect.

\textbf{Imbalance Classification:} Effective classification in the presence of imbalanced data is an important research field in supervised and semi-supervised learning, given that a significant class imbalance is commonly found in numerous real-world contexts \citep{herland2018big,rao2006data,cieslak2006combating} such as fraud detection in finance applications \citep{wei2013effective}. This imbalance often presents challenges, as many classifiers tend to prefer the majority class, sometimes to the extent of completely overlooking the minority class. 
Traditional Class Imbalance Learning methods typically fall into two broad categories: data-level methods and algorithm-level methods. Data-level techniques aim to rebalance the training data itself to mitigate class imbalance. This includes strategies like over-sampling the minority class \citep{chawla2002smote,he2008adasyn,mullick2019generative}, under-sampling the majority class \citep{laurikkala2001improving,mani2003knn}, and hybrid methods that combine both over-sampling and under-sampling \citep{ramentol2012smote,saez2015smote}. On the other hand, algorithm-level methods focus on modifying the learning algorithms to better handle class imbalance. These methods encompass cost-sensitive learning \citep{tang2008svms,zhou2005training}, ensemble learning \citep{seiffert2009rusboost, guo2004learning, chawla2003smoteboost}, and engineering of specialized loss functions \citep{cao2019learning, dong2018imbalanced, cui2019class}.

\section{Proof and details of Section~\ref{sec: gen analysis}}\label{app: sec: gen analysis}

\begin{lemma}[Bound on infinite norm of the symmetric normalized graph filter]\label{lem: bound on inf norm of L}
Consider a graph sample $G$ with adjacency matrix $A$. For the symmetric normalized graph filter, we have $\norm{\tilde{D}^{-1/2}(A+I)\tilde{D}^{-1/2}}_{\infty}[k]\leq \sqrt{\frac{\deg^{\max}_k+1}{\deg^{\min}+1}}$ where $\deg^{\max}_k$ is the maximum degree of graph $G$ with adjacency matrix $A$ among the nodes in $k$-th class. 
\end{lemma}
\begin{proof}
Recall that $\tilde{A}=A+I$. We have,
\begin{equation}
\begin{split}
\norm{\tilde{L}}_{\infty}[k]&=
\norm{\tilde{D}^{-1/2}\tilde{A}\tilde{D}^{-1/2}}_{\infty}[k]\\
&=\max_{i\in[N_k]}\sum_{j=1}^N \frac{\tilde{A}_{ij}}{\sqrt{\deg_{i}+1}\sqrt{\deg_{j}+1}}\\
&\leq\frac{1}{\sqrt{\deg^{\min}+1}} \max_{i\in[N_k]}\sum_{j=1}^{N} \frac{\tilde{A}_{ij}}{\sqrt{\deg_{i}+1}}\\
&\leq  \sqrt{\frac{\deg^{\max}_k+1}{\deg^{\min}+1}}.
\end{split}
\end{equation}
where $N_k$ is the number nodes in $k$-th class.
\end{proof}
\begin{repproposition}{prop: general bound}[\textbf{Restated}]
    Let $Q_i=(\frac{1}{u_i}+\frac{1}{m_i})$, $S_i=\frac{m_i+u_i}{(m_i+u_i-1/2)(1-1/(2\max(m_i,u_i)))}$  for $i=1,2$. Then, with probability at least $(1-\delta)$ over the choice of the training set from nodes of graphs, for all $h_\theta\in\mathcal{H}$,
    \begin{equation*}
    \begin{split}
         R(\mbZ_u,h_\theta)\leq \sum_{i=1}^2 R_{\gamma}(\mbZ_{m_i},h_{\theta})+\frac{u_i}{u}\Big(\frac{1}{\gamma}\mathfrak{R}_{m_i+u_i}(\mathcal{H})
        +c_0 Q_i\sqrt{\min(m_i,u_i)}+\sqrt{\frac{S_iQ_i}{2}\log(1/\delta)}\Big),
    \end{split}
    \end{equation*}
    where $\mathfrak{R}_{m_i+u_i}(\mathcal{H})$ is the transductive Rademacher complexity for $i$-th class, $R_{\gamma}(\mbZ_{m_i},h_{\theta})$ is the empirical risk based on $\gamma$-margin loss for $i$-th class  and $c_0=\sqrt{\frac{32\ln(4e)}{3}}$.
\end{repproposition}
\begin{proof}
    Let's decompose the true risk as follows,
    \begin{equation}\label{eq: v1}
        R(\mbZ_u,h_\theta)=\sum_{i=1}^2\frac{u_i}{u} R(\mbZ_{u_i},h_\theta),
    \end{equation}
    where $R(\mbZ_{u_i},h_\theta)$ is true risk of $i$-th class. Inspired by \citep[Corollary~1]{el2009transductive}, we have,
     \begin{equation}\label{eq: v2}
     \begin{split}
           R(\mbZ_{u_i},h_\theta)\leq R_{\gamma}(\mbZ_{m_i},h_{\theta})+\frac{1}{\gamma}\mathfrak{R}_{m_i+u_i}(\mathcal{H})
        +c_0 Q_i\sqrt{\min(m_i,u_i)}+\sqrt{\frac{S_iQ_i}{2}\log(1/\delta)}.
     \end{split}
    \end{equation}
    Combining \eqref{eq: v2} with \eqref{eq: v1}, we have,
    \begin{equation}
        \begin{split}
            R(\mbZ_u,h_\theta)&=\sum_{i=1}^2\frac{u_i}{u} R(\mbZ_{u_i},h_\theta)\\
            &\leq \sum_{i=1}^2 \frac{u_i}{u} \Big(R_{\gamma}(\mbZ_{m_i},h_{\theta})+\frac{1}{\gamma}\mathfrak{R}_{m_i+u_i}(\mathcal{H})
        +c_0 Q_i\sqrt{\min(m_i,u_i)}+\sqrt{\frac{S_iQ_i}{2}\log(1/\delta)}\Big).
        \end{split}
    \end{equation}
    It completes the proof by considering $\frac{u_i}{u}\leq 1$.
\end{proof}

\begin{repproposition}{prop: TR RC bound}[\textbf{Restated}]
    Given Assumptions~\ref{Ass: Feature node bounded} , \ref{ass: bounded param} and \ref{ass: bounded activation function}, then the following upper bound holds on the transductive Rademacher complexity of node classification,
    \begin{equation*}
        \mathfrak{R}_{m_k+u_k}(\mathcal{H})\leq \frac{B_f (\sqrt{2\log(2) d}+1)\Pi_{j=1}^{d} U_F(i) \|\mathrm{G}_f(\mbA)[i]\|_{\infty}^{2(d-1)}}{\sqrt{m_k+u_k}}
    \end{equation*}
    where $\|\mathrm{G}_f(\mbA)[k]\|_{\infty}$ is the infinite norm of graph filter among $k$-th class.
\end{repproposition}
\begin{proof}
      Fix $\lambda>0$, to be chosen later.
      Let us denote the output of $q$-th layer of GCN as, 
      \[N_{\theta_q}(\mathbf{X}_{u_k+m_k})=\sigma(\theta_{q-1}\mathrm{G}_f(\mbA)[i,:](\cdots\sigma(\theta_1\mathrm{G}_f(\mbA)[i,:]\mathbf{X}_{u_k+m_k}))).\]
      For the last layer, we do not consider aggregation.
      The Rademacher complexity can be 
    upper bounded as
    \begin{align*}
    \Big(\frac{m_k u_k}{m_k+u_k}\Big)\mathfrak{R}_{m_k+u_k}(\mathcal{H}) &= \E_{\boldsymbol{\epsilon}} 
    \sup_{N_{\theta_{d-1}},\theta_d} \sum_{i=1}^{m_k+u_k} 
    \epsilon_i \theta_d \sigma(\theta_{d-1}\mathrm{G}_f(\mbA)[i,:]N_{\theta_{d-1}}(\mathbf{X}_{u_k+m_k}))  \\
    &\leq \frac{1}{\lambda} \log \E_{\boldsymbol{\epsilon}} \sup 
    \exp\left( \lambda \sum_{i=1}^{m_k+u_k} 
    \epsilon_i \theta_d \sigma(\theta_{d-1}\mathrm{G}_f(\mbA)[i,:]N_{\theta_{d-1}}(\mathbf{X}_{u_k+m_k}))\right) \\
    &\leq \frac{1}{\lambda} \log \E_{\boldsymbol{\epsilon}} \sup \exp  
    \left( U_F(d)\cdot 
    \left\|\lambda \sum_{i=1}^{m_k+u_k} \epsilon_i 
    \sigma\big(\theta_{d-1}\mathrm{G}_f(\mbA)[i,:]N_{\theta_{d-1}}(\mathbf{X}_{u_k+m_k})\big) \right\| \right).
    \end{align*}
    We write this last expression as
    \begin{align*}
    &\frac{1}{\lambda} \log \E_{\boldsymbol{\epsilon}} \sup_{f, 
        \|\theta_{d-1}\|_F \leq U_F(d-1)} 
    \exp  \left( U_F(d)\cdot \lambda \left\| \sum_{i=1}^{m_k+u_k} \epsilon_i 
    \sigma(\theta_{d-1} \mathrm{G}_f(\mbA)[i,:]N_{\theta_{d-1}}(\mathbf{X}_{u_k+m_k})) \right\| \right)\\
    &\leq \frac{1}{\lambda} \log \left( 2\cdot 
    \E_{\boldsymbol{\epsilon}} \sup_{f} \exp \left( 
    U_F(d)\cdot U_F(d-1) \cdot  \|\mathrm{G}_f(\mbA)\|_{\infty}^{2} \lambda \left\| \sum_{i=1}^{m_k+u_k} \epsilon_i 
    N_{\theta_{d-1}}(\mathbf{X}_{u_k+m_k}) \right\| \right) \right).
    \end{align*}
     Repeating the process, we arrive at
    \begin{align}
     \Big(\frac{m_k u_k}{m_k+u_k}\Big)\hat{\Rcal}_{m_i+u_i}(\Hcal) &\leq \frac{1}{\lambda} \log \left(2^d \cdot 
    \E_{\boldsymbol{\epsilon}} 
    \exp \left( M \lambda \left\| \sum_{i=1}^{m_k+u_k} \epsilon_i \mathbf{X}_{i} \right\| 
    \right) \right).
    \label{eq:softmax_bd}
    \end{align}
    where $V= \|\mathrm{G}_f(\mbA)\|_{\infty}^{2(d-1)} \prod_{j=1}^d U_F(j)$.
    Define a random variable
    $$Z = V\cdot\left\|\sum_{i=1}^{m}\epsilon_i \mathbf{X}_{i}\right\|,$$
    (random as a function of the random variables 
    $\epsilon_1,\ldots,\epsilon_m$). Then
    \begin{align}
    \label{eq:softmax_bd2}
    \frac{1}{\lambda} \log \left\{ 2^d \cdot \E \exp \lambda Z \right\} 
    &= \frac{d\log(2)}{\lambda} +  \frac{1}{\lambda} \log \left\{ \E 
    \exp 
    \lambda (Z-\E Z) \right\} + \E Z.
    \end{align}
    By Jensen's inequality, $\E[Z]$ can be upper bounded by using the Rademacher complexity random variables
    \[
    V\sqrt{\E_{\boldsymbol{\epsilon}}\left[\left\|\sum_{i=1}^{m}\epsilon_i\mathbf{X}_{i}\right\|^2\right]}
    ~=~
    V\sqrt{\E_{\boldsymbol{\epsilon}}\left[\sum_{i,i'=1}^{m}\epsilon_i\epsilon_{i'}\mathbf{X}_{i}^\top\bx_{i'}\right]}
    ~=~
    2Vp\sqrt{\sum_{i=1}^{m}\norm{\mathbf{X}_{i}}^2},
    \]
    where $p=\frac{m_k u_k}{(m_k+u_k)^2}$. To handle the $ \log \left\{ \E 
    \exp 
    \lambda (Z-\E Z) \right\}$ term in \eqref{eq:softmax_bd2}, note that 
    $Z$ is 
    a deterministic function of the i.i.d. random variables 
    $\epsilon_1,\ldots,\epsilon_m$, and satisfies
    \begin{align*}
    Z(\epsilon_1,\ldots,\epsilon_i, \ldots,\epsilon_m) - 
    Z(\epsilon_1,\ldots,-\epsilon_i, \ldots,\epsilon_m) \leq 2V \|\mathbf{X}_{i}\|~.
    \end{align*}
    This means that $Z$ satisfies a bounded-difference condition, which by 
    the 
    proof of Theorem 6.2 in \citep{boucheron2013concentration}, implies 
    that 
    $Z$ is sub-Gaussian, with variance factor
    $$v = \frac{1}{4} \sum_{i=1}^{m_k+u_k} (2V\|\mathbf{X}_{i}\|)^2 = 4 p^2 V^2\sum_{i=1}^{m_k+u_k} 
    \|\mathbf{X}_{i}\|^2,$$
    and satisfies
    $$ \frac{1}{\lambda} \log \left\{ \E \exp \lambda (Z-\E Z) \right\} 
    \leq 
    \frac{1}{\lambda} \frac{\lambda^2 4 V^2 p^2 \sum_{i=1}^{m_k+u_k} \|\mathbf{X}_{i}\|^2}{2} =  
    \frac{\lambda 4 V^2 p^2  \sum_{i=1}^{m_k+u_k} \|\mathbf{X}_{i}\|^2}{2}.$$
    Choosing $\lambda = \frac{\sqrt{2\log(2)d}}{2pV \sqrt{\sum_{i=1}^{m_k+u_k} \| 
    \mathbf{X}_{i} 
            \|^2}}$ 
    and using the above, we get that \eqref{eq:softmax_bd} can be upper 
    bounded 
    as follows:
    \begin{align*}
    \frac{1}{\lambda} \log \left\{ 2^d \cdot \E \exp \lambda Z \right\} 
    &\leq \E Z + \sqrt{2\log(2)d} \cdot 2p V\sqrt{\sum_{i=1}^{m_k+u_k} \|\mathbf{X}_{i}\|^2} 
    \\&\leq 2pV 
    \left(\sqrt{2\log(2)d}+1\right)\sqrt{\sum_{i=1}^{m_k+u_k} \|\mathbf{X}_{i}\|^2}\\
     &\leq2pV 
    \left(\sqrt{2\log(2)d}+1\right)B_f\sqrt{m_k+u_k }\\
    &= 2pV 
    \left(\sqrt{2\log(2)d}+1\right)B_f\sqrt{m_k+u_k }
    ~,
    \end{align*}
    from which the result follows.
\end{proof}
\begin{reptheorem}{thm: main result}[\textbf{Restated}]
     Under the same Assumptions in Proposition~\ref{prop: general bound} and Proposition~\ref{prop: TR RC bound}, the following upper bound holds on population risk for imbalance transductive node classification under a GCN model,
     \begin{equation*}
    \begin{split}
         R(\mbZ_u,h_\theta)&\leq \sum_{i=1}^2 R_{\gamma}(\mbZ_{m_i},h_{\theta})+\frac{u_i}{u}\Big(\frac{B_f (\sqrt{2\log(2) d}+1)\Pi_{j=1}^{d} U_F(j) }{\gamma\sqrt{(m_i+u_i)}}\times \Bigg(\frac{\degmax_i+1}{\deg^{\min}+1}\Bigg)^{d-1}
        \\&\qquad +c_0 Q_i\sqrt{\min(m_i,u_i)}+\sqrt{\frac{S_iQ_i}{2}\log(1/\delta)}\Big),
    \end{split}
    \end{equation*}
where $U_F(i)$ is the maximum of Frobenius norm of $i$-th layer parameter. 
\end{reptheorem}
\begin{proof}
    The proof follows directly from combining Proposition~\ref{prop: general bound} with Proposition~\ref{prop: TR RC bound} and substituting $\|\mathrm{G}_f(\mbA)\|_{\infty}$ for the GCN model using Lemma~\ref{lem: bound on inf norm of L}.
\end{proof}
\clearpage
\section{UPL Algorithm Details}\label{app: UPL alg detail}
\begin{algorithm}[H]
\caption{Uncertainty-aware Pseudo-labeling  (UPL) Algorithm }
\label{alg1}
\begin{scriptsize} 
\begin{algorithmic}[1] 
\STATE \textbf{Input:} Graph $G(X,V,E,Y)$, set of known train labels $\mathcal{Y}_{t}$, Number of classes $C$, set of edges $\mathcal{E}$, Number of Iterations $k$
\STATE \textbf{Output:} Model parameters
\STATE Initialize empty pseudo labels set $\mathcal{Y}_{p}$
\FOR{ $1$ to $k$}
    \STATE Initialize new model
    \STATE Train  $model$ on $\mathcal{Y}_{t}\cup \mathcal{Y}_{p}$ utilizing BS
    \STATE Initialize results set $R$ with known labels
    \FOR{$i \leftarrow 1 $ to $ t$}
        \STATE $G_{\mathrm{Aug}} \leftarrow $ Selective Edge Removal on $G$ for $S_k$ times
        \STATE  $R_i \leftarrow $ \{Inference Model on $G_{\mathrm{Aug}}$\} \% Add output to $R$
    \ENDFOR
    \STATE $P \leftarrow $ \{Inference Model on $G$\} \% probs is the class-wise probability
    \STATE Initialize the set $\mathcal{Y}^{\mathrm{new}}$ for new pseudo labels
    \FOR{$v$ in $G(V)$}
        \IF{$v$ not in $\mathcal{Y}_{t}$}    
            \IF{$\eta_l<\max(P_k^v)<\eta_u$}
                \IF{$U(v)<Q_\alpha$}
                    \STATE Add $v$ and $\mathrm{argmax}(P_k^v)$ to $\mathcal{Y}^{\mathrm{new}}$ as node and pseudo label
                \ENDIF
            \ENDIF
        \ENDIF
    \ENDFOR
    \STATE $\mathcal{Y}^{p}  \leftarrow \mathcal{Y}^{\mathrm{new}}$    
\ENDFOR
\STATE \textbf{return}  $best\_model$ parameters
\end{algorithmic}
\end{scriptsize}
\end{algorithm}

\section{Experiment Details}\label{app: Exp details}


\textbf{Datasets:} We summarized the dataset statistics in Table \ref{tab:datasets}.  Note that in the Amazon Computers and Amazon Photos datasets \cite{shchur2019pitfallsgraphneuralnetwork}, nodes represent products and edges denote co-purchasing relationships between them.
 
\begin{table}[htp]
\centering
\caption{\label{tab:datasets}Dataset statistics.}
\begin{tabular}{l c r r r r}
\toprule
\textbf{Dataset} & \textbf{Nodes} & \textbf{Edges} & \textbf{Classes} & \textbf{Features}   \\
\midrule
CiteSeer &  3,327 & 4,732 & $6$ & 3,703  \\
Cora &  2,708 & 5,429 & $7$ & 1,433  \\
Pubmed &  19,717 & 44,338 & $3$ & 500  \\
Chameleon &  2277 & 36101 & 5 & 2325\\
Squirrel &  5201 & 217073 & 5 & 2089\\ 
Wisconsin & 251 & 515 & 5 &1703\\
Computers & 13,381 & 245,778 & 10 & 767\\
Photo & 7,650 & 119,081 & 8 & 745\\
\bottomrule
\end{tabular}
\end{table}

\textbf{Hyperparameters:} To attain the best results we have fine-tuned three hyperparameters for our model. These hyperparameters and their tune range along with hyperparameters from some of the previous have been displayed in Table \ref{tb:hyperparameters}.
\begin{table}[h]
\centering
\caption{Tuned hyperparameters used in our work compared with those from previous methods}
\label{tb:hyperparameters}
\vspace{-0.5em}
\resizebox{0.4\columnwidth}{!}{\begin{tabular}{c|c}
\toprule
\textbf{Methods} & \textbf{Hyperparameters Range} \\ \cline{1-2}
GraphENS & \makecell{$\lambda \in \beta(2,2)$ \\
                     $k \in \{1,5,10\}$ \\
                     $\tau \in \{1,2\}$ \\
                     warm up $\in \{1,5\}$} \\ \cline{1-2}
TAM & \makecell{$\alpha \in \{0.25,0.5,1.5,2.5\}$ \\
                $\beta  \in \{0.125,0.25,0.5\}$ \\
                $\phi   \in \{0.8,1.2\} $ } \\ \cline{1-2}
Unreal & \makecell{$K \in \{100,300,500,700,900\}$ \\
                   $T \in \{40,50,60,70,80,100\}$  \\
                   $\alpha \in \{4,6,8,10\}$ \\
                   $p \in \{0.5,0.75,0.98\}$ \\
                   DGIN threshold $\in \{0.25,0.5,0.75,1.00\}$} \\ \cline{1-2}
UPL & \makecell{$\eta_l \in \{0.25 + 0.05n \mid n \leq 9, n\in \mathbb{N} \cup \{0\}\}$ \\ 
                $\eta_u \in \{min(\eta_l + 0.25n,1.00) \mid n\leq3, n \in \mathbb{N}\}$ \\
                $Q_{\alpha} \in \{Q_{0.7},Q_{0.8},Q_{0.9}\}$
                    } \\

\bottomrule
\end{tabular}}
\end{table}
\section{Experiments Setup}\label{app:Exp_set}
UPL contains two tunable hyperparameters used for the pseudo labels threshold and three fixed ones which have approximated using a few experiments. Other extra hyperparameters have been adjusted according to \citep{song2022tam}. We will discuss the setting for these parameters in the following sections.
\subsection{UPL}\label{app:UPL}
We employ two thresholds for pseudo-labeling. As mentioned in Table \ref{tb:hyperparameters}, $\eta_l$ and $\eta_u$ are chosen respectively from $\{0.25 + 0.05n \mid 0\leq n \leq 8, n\in \mathbb{N}\}$ and $\eta_u \in \{min(\eta_l + 0.25n,1) \mid n\leq3, n \in \mathbb{N}\}$. Both $S_k$ and $t$ are fixed to 100, and the quantile for uncertainty-based node selection is chosen from $\{Q_{0.7},Q_{0.8},Q_{0.9}\}$.
\subsection{BalancedSoftmax}\label{app:BalancedSoftmax}
Another part of UPL is BalancedSoftmax (BS) presented in \citep{ren2020balanced}. We have used the same values found in \citep{song2022tam} to prevent increasing the number of hyperparameters in the UPL algorithm.
\subsection{Other Details}\label{app:Other_details}
More minor details concerning the experiments have been explained in the following sections.

\textbf{Imbalance Ratio:} The distribution of training sets for each dataset can be seen in Table \ref{tb:datastat}. All datasets follow a 10:1 imbalance ratio, except for Wisconsin, which has an imbalance ratio ($\rho$) of 11.63. By analyzing Table \ref{tb:datastat}, it is evident from the results in Tables \ref{tb:apendix_homo} and \ref{tb:apendix_hetero} that our model generally performs significantly better when the node with the maximum degree is present in one of the majority classes.

\begin{table}[h] 
\center
\caption{\small Number of training nodes per class}
\begin{small}
\setlength{\tabcolsep}{3pt} 
\begin{tabular}{l|cccccccccc}
\toprule
    \textbf{Dataset} ($\rho=10$) & $\mathbf{C}_0$ & $\mathbf{C}_1$ & $\mathbf{C}_2$ & $\mathbf{C}_3$ & $\mathbf{C}_4$ & $\mathbf{C}_5$ & $\mathbf{C}_6$ & $\mathbf{C}_7$ & $\mathbf{C}_8$ & $\mathbf{C}_9$ \\
    \hline
    Cora  & 20 & 20 & 20 & 20 & 2 & 2 & 2 & - & - & - \\    
    CiteSeer & 20 & 20 & 20 & 2 & 2 & 2 & - & - & - & - \\
    PubMed & 20 & 20 & 2 &- &- &- &- & - & - & - \\
    Chameleon & 225 & 220 & 218 & 22 & 22 & - & - & - & - & - \\
    Squirrel & 487 & 494 & 501 & 50 & 50 & - & - & - & - & - \\ 
    Wisconsin ($\rho=11.63$) & 4 & 38 & 50 & 5 & 5 & - & - & - & - & - \\ 
    Computers & 20 & 20 & 20 & 20 & 20 & 2 & 2 & 2 & 2 & 2 \\ 
    Photo & 20 & 20 & 20 & 20 & 2 & 2 & 2 & 2 & - & - \\ 
\bottomrule
\end{tabular}
\end{small}
\label{tb:datastat}
\end{table}

\begin{table}[h] 
\center
\caption{\small Size of maximum degree node per class}
\begin{small}
\setlength{\tabcolsep}{3pt} 
\begin{tabular}{l|ccccccc}
\toprule
    \textbf{Dataset}  & $\mathbf{C}_0$ & $\mathbf{C}_1$ & $\mathbf{C}_2$ & $\mathbf{C}_3$ & $\mathbf{C}_4$ & $\mathbf{C}_5$ & $\mathbf{C}_6$ \\
    \hline

    Cora  & 36 & 78 & 168 & 74 & 40 & 23 & 31 \\    
    CiteSeer & 15 & 21 & 99 & 18 & 51 & 17 &-  \\
    PubMed & 80 & 171 & 154 &- &- &- &- \\
    Chameleon & 207 & 269 & 322 & 531 & 732 & - & - \\
    Squirrel & 1041 & 1054 & 1081 & 1060 & 1904 & - & -  \\ 
    Wisconsin & 5 & 122 & 10 & 13 & 11 & - & -  \\ 
\bottomrule
\end{tabular}
\end{small}
\label{tb:datastat_degree}
\end{table}

\paragraph{Hardware and setup:} All of our Experiments were conducted on a single server containing an Nvidia RTX 3090 GPU, AMD Ryzen 5 3600x CPU @ 3.80GHz, and 48GB of RAM. 


\paragraph{Training:} To choose the training mask, we use the training masks provided by Pytorch Geometric. We train each model for 1000 epochs with an early stopping patience of 100. For each dataset $\eta_l$ is chosen from $(0.25,0.90)$ and the value of $\eta_u$ is chosen from $(\eta_u,1.00)$.
We also perform each experiment with 10 iterations of pseudo-labeling.

\paragraph{Optimizer:} 
we use PyTorch Adaptive Moment Estimation (ADAM) as the optimizer. Also, $\ell_2$ regularization with weight decay of $5e^{-4}$ for the first layer and dropout in some layers are used to prevent over-fitting. We use a learning rate of 0.01.

\section{More Experiments}\label{App: more exp}

In this section, we present the performance outcomes for three prominent GNN architectures: Graph Convolutional Network (GCN), Graph Attention Network (GAT), and Graph SAGE (SAGE). These models were evaluated across six diverse datasets with three different imbalance ratios to assess their effectiveness in various scenarios. The results for the imbalance ratio of 10 are detailed in Table \ref{tb:apendix_homo} and Table \ref{tb:apendix_hetero}.

Tables \ref{tb:Homo_ir_5} and \ref{tb:Hetero_ir_5} present the performance evaluation of UPL with an imbalance ratio of 5, while Tables \ref{tb:Homo_ir_2} and \ref{tb:Hetero_ir_2} show the corresponding results for a lower imbalance ratio of $\rho=2$.
\begin{table}[ht!]
\caption{\small Experimental results of our regularization methods and other baselines on three class-imbalanced node classification benchmark datasets: Chameleon, Squirrel and Wisconsin. We report averaged balanced accuracy (bAcc.) and F1-score with the standard errors for 20 repetitions on three representative GNN architectures, i.e., GCN, GAT and Sage.}
\begin{center}
\begin{scriptsize}
\setlength{\columnsep}{1pt}%
\resizebox{\linewidth}{!}{
\begin{tabular}{@{\extracolsep{1pt}}rlcc|cc|cc@{}}
\toprule 
 & \multirow{1}{*}{\textbf{Dataset}} & \multicolumn{2}{c}{Cora} & \multicolumn{2}{c}{CiteSeer} & \multicolumn{2}{c}{PubMed}  \\ 
\cline{2-8}\rule{0pt}{2.2ex}
& \textbf{Imbalance Ratio ($\rho=10$)} & bAcc. & F1 & bAcc. & F1 & bAcc. & F1 \\
\cline{2-8}
\rule{0pt}{2.5ex}  
\multirow{8}{*}{\rotatebox{90}{GCN}} 
                    \rule{0pt}{2ex}
                     & Vanilla 
                     & 53.58 \tiny{$\pm 0.70$} & 48.52 \tiny{$\pm 1.09$}
                     & 35.29 \tiny{$\pm 0.34$} & 22.62 \tiny{$\pm 0.39$}
                     & 62.22 \tiny{$\pm 0.35$} & 48.99 \tiny{$\pm 1.02$} \\

                     & Re-Weight
                     & 65.52 \tiny{$\pm 0.84$}& 65.54 \tiny{$\pm 1.20$}
                     & 44.52 \tiny{$\pm 1.22$}& 38.85 \tiny{$\pm 1.62$}
                     & 70.17 \tiny{$\pm 1.25$}& 66.37 \tiny{$\pm 1.73$}
                     \\
                    
                     & PC Softmax 
                     & 67.79 \tiny{$\pm 0.92$}& 67.39 \tiny{$\pm 1.08$}
                     & 49.81 \tiny{$\pm 1.12$}& 45.55 \tiny{$\pm 1.26$}
                     & 70.20 \tiny{$\pm 0.60$}& 68.83 \tiny{$\pm 0.73$}
                     \\
                     
                     & DR-GCN 
                     & 60.17 \tiny{$\pm 0.83$}& 59.31 \tiny{$\pm 0.97$}
                     & 42.64 \tiny{$\pm 0.75$}& 38.22 \tiny{$\pm 1.22$}
                     & 65.51 \tiny{$\pm 0.81$}& 64.95 \tiny{$\pm 0.53$}
                     \\
                     & GraphSMOTE 
                     & 62.66 \tiny{$\pm 0.83$} & 61.76 \tiny{$\pm 0.96$}
                     & 34.26 \tiny{$\pm 0.89$} & 28.31 \tiny{$\pm 1.48$}
                     & 68.94 \tiny{$\pm 0.89$} & 64.17 \tiny{$\pm 1.43$}\\

                     & BalancedSoftmax 
                     & 67.75 \tiny{$\pm 1.09$} & 66.47 \tiny{$\pm 1.27$}
                     & 52.09 \tiny{$\pm 1.61$} & 48.55 \tiny{$\pm 2.01$}
                     & 71.78 \tiny{$\pm 0.82$} & 71.18 \tiny{$\pm 0.94$} \\

                     & GraphENS + TAM 
                     & 69.42 \tiny{$\pm 0.90$} & 68.70 \tiny{$\pm 1.08$}
                     & 53.10 \tiny{$\pm 1.82$} & 50.67 \tiny{$\pm 2.28$}
                     & 71.70 \tiny{$\pm 0.95$} & 69.01 \tiny{$\pm 1.72$} \\
                        
                     & Unreal 
                     & \textbf{78.33} \tiny{$\pm 1.04$} & \textbf{76.44} \tiny{$\pm 1.06$}
                     & \underline{65.63} \tiny{$\pm 1.38$} & \underline{64.94} \tiny{$\pm 1.38$}
                     & \underline{75.35} \tiny{$\pm 1.41$} & \underline{73.65} \tiny{$\pm 1.43$} \\
                     
                     \cline{2-8}
                     
                     & UPL
                     & \underline{76.16} \tiny{$\pm 0.76$} & \underline{74.53} \tiny{$\pm 0.85$}
                    & \textbf{68.18} \tiny{$\pm 0.27$} & \textbf{67.78} \tiny{$\pm 0.25$}
                    & \textbf{77.50} \tiny{$\pm 0.44$} & \textbf{77.19} \tiny{$\pm 0.43$} \\
\cline{2-8}
\noalign{\vskip\doublerulesep
         \vskip-\arrayrulewidth} \cline{2-8}
\rule{0pt}{2.5ex}  
\multirow{8}{*}{\rotatebox{90}{GAT}}
                     & Vanilla 
                     & 49.47 \tiny{$\pm 0.74$} & 45.14 \tiny{$\pm 1.04$}
                     & 33.67 \tiny{$\pm 0.27$} & 21.70 \tiny{$\pm 0.25$}
                     & 59.98 \tiny{$\pm 0.33$} & 47.05 \tiny{$\pm 0.81$} \\
                     
                     & Re-Weight 
                     & 66.72 \tiny{$\pm 0.80$}& 66.52 \tiny{$\pm 1.06$}
                     & 45.59 \tiny{$\pm 1.73$}& 39.43 \tiny{$\pm 2.03$}
                     & 69.13 \tiny{$\pm 1.25$}& 64.81 \tiny{$\pm 1.70$} \\
                    
                     & PC Softmax 
                     & 67.02 \tiny{$\pm 0.65$}& 66.57 \tiny{$\pm 0.89$}
                     & 50.70 \tiny{$\pm 1.73$}& 47.14 \tiny{$\pm 1.85$}
                     & 72.20 \tiny{$\pm 0.49$}& 70.95 \tiny{$\pm 0.82$} \\
                     
                     & DR-GCN 
                     & 59.30 \tiny{$\pm 0.76$}& 57.79 \tiny{$\pm 1.03$}
                     & 44.04 \tiny{$\pm 1.26$}& 39.44 \tiny{$\pm 1.76$}
                     & 69.56 \tiny{$\pm 1.01$}& 68.49 \tiny{$\pm 0.71$} \\
                     
                     & GraphSMOTE 
                     & 56.50 \tiny{$\pm 0.71$} & 54.27 \tiny{$\pm 1.09$}
                     & 44.94 \tiny{$\pm 1.36$} & 41.63 \tiny{$\pm 1.78$}
                     & 62.86 \tiny{$\pm 0.53$} & 53.00 \tiny{$\pm 1.17$}\\
                     & BalancedSoftmax
                     & 66.63 \tiny{$\pm 0.91$} & 65.66 \tiny{$\pm 0.98$}
                     & 52.51 \tiny{$\pm 1.28$} & 49.71 \tiny{$\pm 1.80$}
                     & 71.09 \tiny{$\pm 0.74$} & 68.95 \tiny{$\pm 0.97$} \\
                     & GraphENS + TAM
                     & 69.35 \tiny{$\pm 0.85$} & 68.20 \tiny{$\pm 0.94$}
                     & 53.50 \tiny{$\pm 1.21$} & 51.15 \tiny{$\pm 1.58$}
                     & 71.27 \tiny{$\pm 0.85$} & 68.43 \tiny{$\pm 1.65$} \\

                    & Unreal 
                     & \textbf{78.91} \tiny{$\pm 0.59$} & \textbf{75.99} \tiny{$\pm 0.47$}
                     & \underline{64.10} \tiny{$\pm 1.49$} & \underline{63.44} \tiny{$\pm 1.47$}
                     & \underline{74.68} \tiny{$\pm 1.43$} & \underline{72.78} \tiny{$\pm 0.89$} \\
                     \cline{2-8}
                     & UPL 
                     & \underline{76.28} \tiny{$\pm 0.75$} & \underline{75.55} \tiny{$\pm 0.79$}
                     & \textbf{66.44} \tiny{$\pm 0.59$} & \textbf{65.98} \tiny{$\pm 0.57$}
                     & \textbf{75.27} \tiny{$\pm 0.42$} & \textbf{74.43} \tiny{$\pm 0.61$} \\
\cline{2-8}
\noalign{\vskip\doublerulesep
         \vskip-\arrayrulewidth} \cline{2-8}
\rule{0pt}{2.5ex}  
\multirow{8}{*}{\rotatebox{90}{Sage}} 
                    & Vanilla 
                     & 50.82 \tiny{$\pm 0.47$} & 43.77 \tiny{$\pm 0.75$}
                     & 34.95 \tiny{$\pm 0.19$} & 22.54 \tiny{$\pm 0.24$}
                     & 61.63 \tiny{$\pm 0.63$} & 49.58 \tiny{$\pm 1.24$} \\

                     & Re-Weight
                     & 63.76 \tiny{$\pm 0.98$}& 63.46 \tiny{$\pm 1.22$}
                     & 46.64 \tiny{$\pm 1.92$}& 41.38 \tiny{$\pm 2.76$}
                     & 69.03 \tiny{$\pm 1.17$}& 64.01 \tiny{$\pm 2.18$}
                     \\
                    
                     & PC Softmax 
                     & 64.03 \tiny{$\pm 0.81$}& 63.73 \tiny{$\pm 0.99$}
                     & 50.14 \tiny{$\pm 1.89$}& 47.38 \tiny{$\pm 2.13$}
                     & 71.39 \tiny{$\pm 0.84$} & 70.25 \tiny{$\pm 1.02$}
                     \\
                     
                     & DR-GCN 
                     & 61.05 \tiny{$\pm 1.17$}& 60.17 \tiny{$\pm 1.23$}
                     & 46.00 \tiny{$\pm 0.93$}& 47.73 \tiny{$\pm 1.12$}
                     & 69.23 \tiny{$\pm 0.68$}& 67.35 \tiny{$\pm 0.90$}
                     \\
                     
                     & GraphSMOTE 
                     & 65.37 \tiny{$\pm 0.55$} & 65.13 \tiny{$\pm 0.68$}
                     & 38.94 \tiny{$\pm 1.05$} & 33.60 \tiny{$\pm 1.68$}
                     & 64.15 \tiny{$\pm 0.55$} & 54.00 \tiny{$\pm 1.14$}\\
                     & BalancedSoftmax
                     & 63.91 \tiny{$\pm 0.88$} & 63.01 \tiny{$\pm 1.00$}
                     & 49.45 \tiny{$\pm 1.65$} & 46.08 \tiny{$\pm 2.08$}
                     & 70.20 \tiny{$\pm 0.72$} & 69.60 \tiny{$\pm 0.86$} \\
                     & GraphENS + TAM
                     & 65.36 \tiny{$\pm 0.96$} & 64.63 \tiny{$\pm 1.16$}
                     & 50.37 \tiny{$\pm 1.62$} & 47.41 \tiny{$\pm 1.99$}
                     & 70.48 \tiny{$\pm 0.87$} & 68.65 \tiny{$\pm 1.61$} \\
                    
                    & Unreal 
                     & \textbf{75.99} \tiny{$\pm 0.98$} & \underline{73.63} \tiny{$\pm 1.23$}
                     & \underline{66.45} \tiny{$\pm 0.39$} & \underline{65.83} \tiny{$\pm 0.30$}
                     & \underline{74.78} \tiny{$\pm 1.30$} & \underline{72.80} \tiny{$\pm 0.54$} \\
                     \cline{2-8}
                     & UPL 
                     & \underline{75.06} \tiny{$\pm 0.84$} & \textbf{73.93} \tiny{$\pm 0.86$}
                     & \textbf{67.41} \tiny{$\pm 0.51$} & \textbf{67.06} \tiny{$\pm 0.51$}
                     & \textbf{74.84} \tiny{$\pm 0.42$} & \textbf{74.26} \tiny{$\pm 0.54$} \\

\bottomrule

\end{tabular}
}

\end{scriptsize}
\end{center}
\label{tb:apendix_homo}
\vspace{-0.1in}
\end{table}
\begin{table}[ht!]
\caption{\small Experimental results of our UPL algorithm and other baselines on three class-imbalanced node classification benchmark datasets: Chameleon, Squirrel and Wisconsin. We report averaged balanced accuracy (bAcc.) and F1-score with the standard errors for 20 repetitions on three representative GNN architectures, i.e., GCN, GAT and Sage}
\begin{center}
\begin{scriptsize}
\setlength{\columnsep}{1pt}%
\resizebox{\linewidth}{!}{
\begin{tabular}{@{\extracolsep{1pt}}rlcc|cc|cc@{}}
\toprule 
 & \multirow{1}{*}{\textbf{Dataset}} & \multicolumn{2}{c}{Chameleon} & \multicolumn{2}{c}{Squirrel} & \multicolumn{2}{c}{Wisconsin \tiny{($\rho=11.63$)}}  \\ 
\cline{2-8}\rule{0pt}{2.2ex}
& \textbf{Imbalance Ratio ($\rho=10$)} & bAcc. & F1 & bAcc. & F1 & bAcc. & F1 \\
\cline{2-8}
\rule{0pt}{2.5ex}  
\multirow{8}{*}{\rotatebox{90}{GCN}}

                     & Vanilla 
                     & 57.01 \tiny{$\pm 0.53$} & 55.61 \tiny{$\pm 0.67$}
                     & 38.07 \tiny{$\pm 0.57$} & 34.81 \tiny{$\pm 0.69$}
                     & 34.51 \tiny{$\pm 2.31$} & 31.45 \tiny{$\pm 2.66$} \\
                     
                     & Re-Weight
                     & 36.07 \tiny{$\pm 0.87$}& 35.61 \tiny{$\pm 0.81$}
                     & 26.92 \tiny{$\pm 0.53$}& 25.04 \tiny{$\pm 0.59$}
                     & 44.13 \tiny{$\pm 3.08$}& 40.74 \tiny{$\pm 3.27$} \\
                    
                     & PC Softmax 
                     & 36.86 \tiny{$\pm 1.04$}& 36.24 \tiny{$\pm 1.01$}
                     & 26.49 \tiny{$\pm 0.59$}& 25.73 \tiny{$\pm 0.49$}
                     & 30.90 \tiny{$\pm 3.10$}& 28.15 \tiny{$\pm 2.16$} \\
                     
                     & DR-GCN 
                     & 33.34 \tiny{$\pm 0.81$}& 29.60 \tiny{$\pm 0.79$}
                     & 23.34 \tiny{$\pm 0.43$}& 18.20 \tiny{$\pm 0.49$}
                     & 29.44 \tiny{$\pm 1.36$}& 27.08 \tiny{$\pm 1.37$} \\
                     
                     & GraphSMOTE 
                     & 52.38 \tiny{$\pm 0.77$}& 50.03 \tiny{$\pm 0.76$}
                     & 41.30 \tiny{$\pm 0.47$}& 39.07 \tiny{$\pm 0.60$}
                     & 45.36 \tiny{$\pm 4.21$}& 40.91 \tiny{$\pm 4.39$} \\
                     & BalancedSoftmax
                     & \underline{58.60} \tiny{$\pm 1.06$} & \underline{58.20} \tiny{$\pm 1.00$}
                     & 41.81 \tiny{$\pm 0.64$} & 41.27 \tiny{$\pm 0.71$}
                     & \underline{44.44} \tiny{$\pm 3.46$} & \underline{41.29} \tiny{$\pm 2.85$}\\
                     & GraphENS + TAM
                     & 40.98 \tiny{$\pm 0.85$} & 38.93 \tiny{$\pm 0.84$}
                     & 27.98 \tiny{$\pm 0.90$} & 24.04 \tiny{$\pm 0.70$}
                     & 41.38 \tiny{$\pm 2.78$} & 34.56 \tiny{$\pm 2.22$}\\

                     & Unreal
                     & 58.18 \tiny{$\pm 0.94$} & 56.98 \tiny{$\pm 1.03$}
                     & \textbf{45.65} \tiny{$\pm 1.08$} & \textbf{44.78} \tiny{$\pm 1.07$}
                     & 41.26 \tiny{$\pm 5.19$} & 32.14 \tiny{$\pm 2.15$}\\

                     \cline{2-8}
                     & UPL 
                     & \textbf{61.19} \tiny{$\pm 0.66$} & \textbf{60.71} \tiny{$\pm 0.71$}
                     & \underline{43.34} \tiny{$\pm 0.25$} & \underline{42.82} \tiny{$\pm 0.07$}
                     & \textbf{48.70} \tiny{$\pm 2.88$} & \textbf{42.72} \tiny{$\pm 2.38$} \\
\cline{2-8}
\noalign{\vskip\doublerulesep
         \vskip-\arrayrulewidth} \cline{2-8}
\rule{0pt}{2.5ex}  
\multirow{8}{*}{\rotatebox{90}{GAT}}
                     & Vanilla 
                     & 48.17 \tiny{$\pm 0.91$} & 44.01 \tiny{$\pm 1.48$}
                     & 29.12 \tiny{$\pm 0.68$} & 22.05 \tiny{$\pm 0.58$}
                     & 33.25 \tiny{$\pm 1.83$} & 31.50 \tiny{$\pm 2.29$} \\
                     
                     & Re-Weight 
                     & 35.72 \tiny{$\pm 0.65$}& 34.19 \tiny{$\pm 0.74$}
                     & 25.79 \tiny{$\pm 0.52$}& 24.32 \tiny{$\pm 0.62$}
                     & 42.15 \tiny{$\pm 2.33$}& 37.66 \tiny{$\pm 2.27$}
                     \\
                     & PC Softmax 
                     & 38.32 \tiny{$\pm 0.88$}& 37.46 \tiny{$\pm 0.84$}
                     & 26.52 \tiny{$\pm 0.31$}& 25.71 \tiny{$\pm 0.44$}
                     & 41.89 \tiny{$\pm 3.95$}& 38.03 \tiny{$\pm 3.35$}
                     \\
    
                     & DR-GCN 
                     & 34.84 \tiny{$\pm 0.72$}& 31.53 \tiny{$\pm 0.86$}
                     & 24.69 \tiny{$\pm 0.46$}& 21.81 \tiny{$\pm 0.42$}                
                     & 33.93 \tiny{$\pm 2.34$}& 31.75 \tiny{$\pm 2.50$}
                     \\
                     
                     & GraphSMOTE 
                     & 40.72 \tiny{$\pm 0.83$}& 33.44 \tiny{$\pm 1.28$}
                     & 27.10 \tiny{$\pm 0.49$}& 26.63 \tiny{$\pm 0.63$}
                     & 40.77 \tiny{$\pm 2.24$}& \underline{38.96} \tiny{$\pm 2.48$} \\
                     & BalancedSoftmax
                     & 54.37 \tiny{$\pm 0.87$} & 53.91 \tiny{$\pm 0.89$}
                     & 33.73 \tiny{$\pm 1.11$} & 32.82 \tiny{$\pm 1.12$}
                     & \underline{41.05} \tiny{$\pm 2.87$} & 38.65 \tiny{$\pm 2.48$} \\
                     & GraphENS + TAM
                     & 49.13 \tiny{$\pm 0.73$} & 48.61 \tiny{$\pm 0.71$}
                     & 29.01 \tiny{$\pm 0.41$} & 26.99 \tiny{$\pm 0.67$}
                     & 39.34 \tiny{$\pm 3.36$} & 35.11 \tiny{$\pm 2.64$} \\
                    & Unreal
                     & \underline{58.79} \tiny{$\pm 0.54$} & \underline{57.91} \tiny{$\pm 0.44$}
                     & \underline{43.02} \tiny{$\pm 1.18$} & \underline{41.57} \tiny{$\pm 1.15$}
                     & 38.33 \tiny{$\pm 5.25$} & 31.61 \tiny{$\pm 2.92$}\\
                     \cline{2-8}
                     & UPL 
                     & \textbf{58.87} \tiny{$\pm 0.48$} & \textbf{58.72} \tiny{$\pm 0.44$}
                      & \textbf{43.62} \tiny{$\pm 0.47$} & \textbf{43.07} \tiny{$\pm 0.48$}
                     & \textbf{48.35} \tiny{$\pm 3.16$} & \textbf{43.34} \tiny{$\pm 2.49$} \\
\cline{2-8}
\noalign{\vskip\doublerulesep
         \vskip-\arrayrulewidth} \cline{2-8}
\rule{0pt}{2.5ex}  
\multirow{8}{*}{\rotatebox{90}{Sage}}
                    & Vanilla 
                     & 45.61 \tiny{$\pm 0.41$} & 40.52 \tiny{$\pm 0.71$}
                     & 27.55 \tiny{$\pm 0.41$} & 21.43 \tiny{$\pm 0.29$}
                     & 50.08 \tiny{$\pm 3.41$} & 49.39 \tiny{$\pm 4.14$} \\

                     & Re-Weight & 36.49 \tiny{$\pm 1.21$}& 34.84 \tiny{$\pm 1.30$}
                     & 29.83 \tiny{$\pm 0.59$}& 25.88 \tiny{$\pm 0.42$}
                     & 68.13 \tiny{$\pm 3.19$}& 63.45 \tiny{$\pm 2.27$}
                     \\
                    
                     & PC Softmax & 40.71 \tiny{$\pm 0.82$}& 39.95 \tiny{$\pm 0.98$}
                     & 29.23 \tiny{$\pm 0.50$}& 28.19 \tiny{$\pm 0.54$}
                     & \textbf{70.57} \tiny{$\pm 3.34$}& \textbf{67.13} \tiny{$\pm 2.91$}
                     \\
                     
                     & DR-GCN & 39.58 \tiny{$\pm 0.58$}& 38.37 \tiny{$\pm 0.72$}
                     & 28.78 \tiny{$\pm 0.50$}& 25.01 \tiny{$\pm 0.70$}
                     & 69.30 \tiny{$\pm 1.99$}& 64.60 \tiny{$\pm 2.00$}
                  
                     \\
                     & GraphSMOTE 
                     & 33.31 \tiny{$\pm 0.63$}& 30.83 \tiny{$\pm 0.67$}
                     & 25.51 \tiny{$\pm 0.43$}& 19.79 \tiny{$\pm 0.49$}
                     & 65.14 \tiny{$\pm 3.84$}& 62.53 \tiny{$\pm 3.40$} \\
                     & BalancedSoftmax
                     & \underline{54.01} \tiny{$\pm 0.54$} & \underline{53.72} \tiny{$\pm 0.52$}
                     & 32.70 \tiny{$\pm 0.50$} & \underline{32.37} \tiny{$\pm 0.51$}
                     & 63.51 \tiny{$\pm 3.67$} & 62.50 \tiny{$\pm 2.97$} \\
                     & GraphENS + TAM
                     & 48.01 \tiny{$\pm 0.97$} & 47.19 \tiny{$\pm 1.05$}
                     & 30.75 \tiny{$\pm 0.51$} & 30.54 \tiny{$\pm 0.53$}
                     & 59.69 \tiny{$\pm 3.04$} & 53.28 \tiny{$\pm 2.61$} \\
                    & Unreal
                     & 50.21 \tiny{$\pm 1.63$} & 49.35 \tiny{$\pm 1.84$}
                     & \textbf{33.43} \tiny{$\pm 0.53$} & 32.30 \tiny{$\pm 0.49$} 
                     & 62.20 \tiny{$\pm 6.83$} & 51.52 \tiny{$\pm 4.14$}\\
                     \cline{2-8}
                     & UPL 
                     & \textbf{54.68} \tiny{$\pm 0.63$} & \textbf{54.22} \tiny{$\pm 0.67$}
                     & \underline{33.19} \tiny{$\pm 0.75$} & \textbf{32.74} \tiny{$\pm 0.75$}
                     & \underline{70.19} \tiny{$\pm 4.50$} & \underline{67.09} \tiny{$\pm 3.78$} \\

\bottomrule

\end{tabular}
}

\end{scriptsize}
\end{center}
\label{tb:apendix_hetero}
\vspace{-0.1in}
\end{table}
\begin{table}[ht!]
\caption{ Experimental results of our algorithm (UPL) compared to baselines for $\rho=5$. We report the averaged balanced accuracy, averaged F1-score and the standard error of each experiment for Cora, CiteSeer and PubMed datasets. All the results have been calculated over 10 repetitions and 10 pseudo-labeling iterations.}
\label{tb:Homo_ir_5}
\begin{center}
\begin{scriptsize}
\setlength{\columnsep}{1pt}%
\resizebox{0.9\linewidth}{!}{
\begin{tabular}{@{\extracolsep{1pt}}rlcc|cc|cc@{}}
\toprule 
 & \multirow{1}{*}{\textbf{Datasets}} & \multicolumn{2}{c}{Cora} & \multicolumn{2}{c}{CiteSeer} & \multicolumn{2}{c}{PubMed}  \\ 
\cline{2-8} \rule{0pt}{2.2ex}
& \textbf{Imbalance Ratio ($\rho=5$)} & bAcc. & F1 & bAcc. & F1 & bAcc. & F1 \\
\cline{2-8}
\rule{0pt}{2.5ex}  
\multirow{4}{*}{\rotatebox{90}{GCN}} 
                    \rule{0pt}{2ex}
                    
            & Vanilla 
& 63.07 \tiny{$\pm 0.76$} & 62.12 \tiny{$\pm 0.98$}
& 39.11 \tiny{$\pm 1.21$} & 30.01 \tiny{$\pm 1.89$}
& 67.37 \tiny{$\pm 0.96$} & 59.63 \tiny{$\pm 1.76$}
\\
& BalancedSoftmax 
& 73.31 \tiny{$\pm 0.76$} & 72.77 \tiny{$\pm 0.68$}
& 58.82 \tiny{$\pm 1.27$} & 57.09 \tiny{$\pm 1.64$}
& \underline{74.29} \tiny{$\pm 0.97$} & \underline{73.60} \tiny{$\pm 0.91$}
\\
& GraphENS + TAM 
& \underline{75.25} \tiny{$\pm 0.82$} & \underline{75.00} \tiny{$\pm 0.61$}
& \underline{60.68} \tiny{$\pm 1.05$} & \underline{59.86} \tiny{$\pm 1.17$}
& 74.18 \tiny{$\pm 0.63$} & 73.21 \tiny{$\pm 0.68$}
\\
& UPL 
& \textbf{78.66} \tiny{$\pm 0.52$} & \textbf{77.62} \tiny{$\pm 0.55$}
& \textbf{68.29} \tiny{$\pm 0.13$} & \textbf{67.92} \tiny{$\pm 0.15$}
& \textbf{77.81} \tiny{$\pm 0.38$} & \textbf{77.49} \tiny{$\pm 0.4$}
\\

\cline{2-8}
\noalign{\vskip\doublerulesep
         \vskip-\arrayrulewidth} \cline{2-8}
\rule{0pt}{2.5ex}  
\multirow{4}{*}{\rotatebox{90}{GAT}}

            & Vanilla 
& 57.19 \tiny{$\pm 0.65$} & 56.18 \tiny{$\pm 0.83$}
& 36.17 \tiny{$\pm 0.59$} & 26.86 \tiny{$\pm 1.12$}
& 65.18 \tiny{$\pm 1.12$} & 56.75 \tiny{$\pm 2.43$}
\\
& BalancedSoftmax 
& 72.48 \tiny{$\pm 0.71$} & 72.05 \tiny{$\pm 0.57$}
& 57.74 \tiny{$\pm 0.90$} & 56.17 \tiny{$\pm 1.11$}
& 73.77 \tiny{$\pm 0.43$} & \underline{72.85} \tiny{$\pm 0.80$}
\\
& GraphENS + TAM 
& \underline{74.27} \tiny{$\pm 0.59$} & \underline{73.19} \tiny{$\pm 0.55$}
& \underline{58.98} \tiny{$\pm 1.11$} & \underline{58.05} \tiny{$\pm 1.41$}
& \underline{73.67} \tiny{$\pm 0.60$} & 72.08 \tiny{$\pm 0.88$}
\\
& UPL 
& \textbf{78.94} \tiny{$\pm 0.81$} & \textbf{77.8} \tiny{$\pm 0.73$}
& \textbf{66.92} \tiny{$\pm 0.26$} & \textbf{66.44} \tiny{$\pm 0.28$}
& \textbf{75.83} \tiny{$\pm 0.45$} & \textbf{75.63} \tiny{$\pm 0.48$}
\\
                    
\cline{2-8}
\noalign{\vskip\doublerulesep
         \vskip-\arrayrulewidth} \cline{2-8}
\rule{0pt}{2.5ex}  
\multirow{4}{*}{\rotatebox{90}{Sage}} 
                    
    & Vanilla 
& 59.30 \tiny{$\pm 0.92$} & 57.30 \tiny{$\pm 1.41$}
& 37.21 \tiny{$\pm 0.92$} & 26.80 \tiny{$\pm 1.56$}
& 65.88 \tiny{$\pm 0.85$} & 57.60 \tiny{$\pm 1.52$}
\\
& BalancedSoftmax 
& 70.41 \tiny{$\pm 0.60$} & 69.84 \tiny{$\pm 0.52$}
& 56.53 \tiny{$\pm 1.19$} & 54.88 \tiny{$\pm 1.48$}
& 72.89 \tiny{$\pm 0.39$} & 71.69 \tiny{$\pm 0.62$}
\\
& GraphENS + TAM 
& \underline{72.93} \tiny{$\pm 0.35$} & \underline{72.43} \tiny{$\pm 0.41$}
& \underline{59.08} \tiny{$\pm 1.19$} & \underline{58.39} \tiny{$\pm 1.27$}
& \underline{74.05} \tiny{$\pm 0.79$} & \underline{73.01} \tiny{$\pm 0.97$}
\\
& UPL 
& \textbf{77.54} \tiny{$\pm 0.56$} & \textbf{76.37} \tiny{$\pm 0.56$}
& \textbf{67.23} \tiny{$\pm 0.45$} & \textbf{66.81} \tiny{$\pm 0.41$}
& \textbf{75.96} \tiny{$\pm 0.58$} & \textbf{75.42} \tiny{$\pm 0.54$}
\\
\bottomrule

\end{tabular}
}

\end{scriptsize}
\end{center}
\vspace{-0.1in}
\end{table}
\begin{table}[ht!]
\caption{Experimental results of our algorithm (UPL) compared to baselines for $\rho=5$. We report the averaged balanced accuracy, averaged F1-score and the standard error of each experiment for Chameleon, Squirrel and Wisconsin datasets. All the results have been calculated over 10 repetitions and 10 pseudo-labeling iterations.}
\label{tb:Hetero_ir_5}
\begin{center}
\begin{scriptsize}
\setlength{\columnsep}{1pt}%
\resizebox{0.9\linewidth}{!}{
\begin{tabular}{@{\extracolsep{1pt}}rlcc|cc|cc@{}}
\toprule 
 & \multirow{1}{*}{\textbf{Datasets}} & \multicolumn{2}{c}{Chameleon} & \multicolumn{2}{c}{Squirrel} & \multicolumn{2}{c}{Wisconsin}  \\ 
\cline{2-8} \rule{0pt}{2.2ex}
& \textbf{Imbalance Ratio ($\rho=5$)} & bAcc. & F1 & bAcc. & F1 & bAcc. & F1 \\
\cline{2-8}
\rule{0pt}{2.5ex}  
\multirow{4}{*}{\rotatebox{90}{GCN}} 
                    \rule{0pt}{2ex}
                    
& Vanilla 
& 62.03 \tiny{$\pm 0.62$} & 61.62 \tiny{$\pm 0.65$}
& 41.48 \tiny{$\pm 0.85$} & 40.38 \tiny{$\pm 0.85$}
& 41.03 \tiny{$\pm 2.00$} & \underline{38.97} \tiny{$\pm 1.91$}
\\
& BalancedSoftmax 
& \underline{63.23} \tiny{$\pm 0.64$} & \underline{62.94} \tiny{$\pm 0.64$}
& \underline{45.36} \tiny{$\pm 0.71$} & \underline{45.04} \tiny{$\pm 0.74$}
& \underline{43.20} \tiny{$\pm 2.90$} & 38.81 \tiny{$\pm 2.19$}
\\
& GraphENS + TAM 
& 43.42 \tiny{$\pm 1.05$} & 41.57 \tiny{$\pm 0.95$}
& 27.78 \tiny{$\pm 0.46$} & 23.29 \tiny{$\pm 0.74$}
& 40.85 \tiny{$\pm 2.93$} & 36.32 \tiny{$\pm 1.78$}
\\
& UPL 
& \textbf{63.7} \tiny{$\pm 0.74$} & \textbf{63.48} \tiny{$\pm 0.71$}
& \textbf{45.49} \tiny{$\pm 0.48$} & \textbf{45.16} \tiny{$\pm 0.48$}
& \textbf{46.51} \tiny{$\pm 2.63$} & \textbf{41.54} \tiny{$\pm 2.27$}
\\

\cline{2-8}
\noalign{\vskip\doublerulesep
         \vskip-\arrayrulewidth} \cline{2-8}
\rule{0pt}{2.5ex}  
\multirow{4}{*}{\rotatebox{90}{GAT}}

& Vanilla 
& 56.11 \tiny{$\pm 0.91$} & 55.45 \tiny{$\pm 0.87$}
& 29.93 \tiny{$\pm 0.49$} & 22.58 \tiny{$\pm 0.42$}
& 37.69 \tiny{$\pm 2.71$} & 34.61 \tiny{$\pm 2.26$}
\\
& BalancedSoftmax 
& \underline{57.55} \tiny{$\pm 0.96$} & \underline{57.50} \tiny{$\pm 0.85$}
& \underline{39.08} \tiny{$\pm 1.38$} & \underline{38.58} \tiny{$\pm 1.43$}
& 39.82 \tiny{$\pm 2.30$} & 37.01 \tiny{$\pm 3.14$}
\\
& GraphENS + TAM 
& 50.02 \tiny{$\pm 0.50$} & 49.68 \tiny{$\pm 0.57$}
& 29.27 \tiny{$\pm 0.40$} & 26.43 \tiny{$\pm 0.88$}
& \underline{45.10} \tiny{$\pm 2.30$} & \underline{37.04} \tiny{$\pm 1.47$}
\\
& UPL 
& \textbf{61.78} \tiny{$\pm 0.8$} & \textbf{61.48} \tiny{$\pm 0.8$}
& \textbf{46.35} \tiny{$\pm 0.58$} & \textbf{46.07} \tiny{$\pm 0.57$}
& \textbf{44.92} \tiny{$\pm 3.18$} & \textbf{41.95} \tiny{$\pm 2.84$}
\\
                    
\cline{2-8}
\noalign{\vskip\doublerulesep
         \vskip-\arrayrulewidth} \cline{2-8}
\rule{0pt}{2.5ex}  
\multirow{4}{*}{\rotatebox{90}{Sage}} 
                    
& Vanilla 
& 49.86 \tiny{$\pm 0.40$} & 48.12 \tiny{$\pm 0.61$}
& 29.41 \tiny{$\pm 0.51$} & 25.90 \tiny{$\pm 0.58$}
& 61.13 \tiny{$\pm 2.03$} & 62.47 \tiny{$\pm 2.21$}
\\
& BalancedSoftmax 
& \underline{56.57} \tiny{$\pm 0.49$} & \underline{56.10} \tiny{$\pm 0.53$}
& \underline{34.27} \tiny{$\pm 0.58$} & \underline{34.04} \tiny{$\pm 0.57$}
& \underline{68.17} \tiny{$\pm 3.38$} & \underline{66.74} \tiny{$\pm 3.50$}
\\
& GraphENS + TAM 
& 49.14 \tiny{$\pm 0.69$} & 48.42 \tiny{$\pm 0.74$}
& 31.90 \tiny{$\pm 0.80$} & 31.39 \tiny{$\pm 0.83$}
& 56.33 \tiny{$\pm 4.20$} & 50.60 \tiny{$\pm 3.27$}
\\
& UPL 
& \textbf{57.42} \tiny{$\pm 0.60$} & \textbf{56.94} \tiny{$\pm 0.60$}
& \textbf{34.49} \tiny{$\pm 0.49$} & \textbf{34.33} \tiny{$\pm 0.51$}
& \textbf{73.61} \tiny{$\pm 2.99$} & \textbf{71.3} \tiny{$\pm 2.89$}
\\
\bottomrule

\end{tabular}
}

\end{scriptsize}
\end{center}
\vspace{-0.1in}
\end{table}
\begin{table}[ht!]
\caption{Experimental results of our algorithm (UPL) compared to baselines for $\rho=2$. We report the averaged balanced accuracy, averaged F1-score and the standard error of each experiment for Cora, CiteSeer and PubMed datasets. All the results have been calculated over 10 repetitions and 10 pseudo-labeling iterations.}
\label{tb:Homo_ir_2}
\begin{center}
\begin{scriptsize}
\setlength{\columnsep}{1pt}%
\resizebox{0.9\linewidth}{!}{
\begin{tabular}{@{\extracolsep{1pt}}rlcc|cc|cc@{}}
\toprule 
 & \multirow{1}{*}{\textbf{Datasets}} & \multicolumn{2}{c}{Cora} & \multicolumn{2}{c}{CiteSeer} & \multicolumn{2}{c}{PubMed}  \\ 
\cline{2-8} \rule{0pt}{2.2ex}
& \textbf{Imbalance Ratio ($\rho=2$)} & bAcc. & F1 & bAcc. & F1 & bAcc. & F1 \\
\cline{2-8}
\rule{0pt}{2.5ex}  
\multirow{4}{*}{\rotatebox{90}{GCN}} 
                    \rule{0pt}{2ex}
                    
& Vanilla 
& 77.09 \tiny{$\pm 0.43$} & 77.18 \tiny{$\pm 0.39$}
& 63.29 \tiny{$\pm 0.93$} & 62.74 \tiny{$\pm 0.98$}
& 76.93 \tiny{$\pm 0.33$} & 75.70 \tiny{$\pm 0.40$}
\\
& BalancedSoftmax 
& 79.46 \tiny{$\pm 0.28$} & 78.58 \tiny{$\pm 0.30$}
& 67.26 \tiny{$\pm 0.49$} & 66.84 \tiny{$\pm 0.53$}
& \underline{78.25} \tiny{$\pm 0.28$} & \underline{77.78} \tiny{$\pm 0.29$}
\\
& GraphENS + TAM 
& \underline{79.88} \tiny{$\pm 0.32$} & \underline{79.15} \tiny{$\pm 0.35$}
& \underline{67.27} \tiny{$\pm 0.34$} & \underline{66.87} \tiny{$\pm 0.36$}
& 76.42 \tiny{$\pm 0.48$} & 76.15 \tiny{$\pm 0.51$}
\\
& UPL 
& \textbf{81.24} \tiny{$\pm 0.31$} & \textbf{80.27} \tiny{$\pm 0.30$}
& \textbf{68.58} \tiny{$\pm 0.29$} & \textbf{68.16} \tiny{$\pm 0.31$}
& \textbf{78.37} \tiny{$\pm 0.21$} & \textbf{77.88} \tiny{$\pm 0.21$}
\\

\cline{2-8}
\noalign{\vskip\doublerulesep
         \vskip-\arrayrulewidth} \cline{2-8}
\rule{0pt}{2.5ex}  
\multirow{4}{*}{\rotatebox{90}{GAT}}

& Vanilla 
& 73.61 \tiny{$\pm 0.49$} & 74.45 \tiny{$\pm 0.49$}
& 60.76 \tiny{$\pm 1.19$} & 59.84 \tiny{$\pm 1.49$}
& 74.43 \tiny{$\pm 0.55$} & 72.14 \tiny{$\pm 0.81$}
\\
& BalancedSoftmax 
& 78.68 \tiny{$\pm 0.28$} & 77.64 \tiny{$\pm 0.35$}
& 65.71 \tiny{$\pm 0.39$} & 65.14 \tiny{$\pm 0.40$}
& \underline{76.54} \tiny{$\pm 0.37$} & 75.49 \tiny{$\pm 0.49$}
\\
& GraphENS + TAM 
& \underline{79.25} \tiny{$\pm 0.27$} & \underline{78.13} \tiny{$\pm 0.28$}
& \underline{66.62} \tiny{$\pm 0.40$} & \underline{66.24} \tiny{$\pm 0.40$}
& 76.28 \tiny{$\pm 0.47$} & \underline{75.87} \tiny{$\pm 0.46$}
\\
& UPL 
& \textbf{80.53} \tiny{$\pm 0.32$} & \textbf{79.76} \tiny{$\pm 0.19$}
& \textbf{67.69} \tiny{$\pm 0.23$} & \textbf{67.31} \tiny{$\pm 0.21$}
& \textbf{77.83} \tiny{$\pm 0.22$} & \textbf{77.04} \tiny{$\pm 0.29$}
\\
                    
\cline{2-8}
\noalign{\vskip\doublerulesep
         \vskip-\arrayrulewidth} \cline{2-8}
\rule{0pt}{2.5ex}  
\multirow{4}{*}{\rotatebox{90}{Sage}} 
                    
& Vanilla 
& 74.44 \tiny{$\pm 0.25$} & 74.63 \tiny{$\pm 0.24$}
& 62.32 \tiny{$\pm 0.70$} & 61.63 \tiny{$\pm 0.80$}
& 75.65 \tiny{$\pm 0.68$} & 74.41 \tiny{$\pm 0.76$}
\\
& BalancedSoftmax 
& 77.56 \tiny{$\pm 0.27$} & 76.61 \tiny{$\pm 0.25$}
& 65.74 \tiny{$\pm 0.48$} & 65.24 \tiny{$\pm 0.49$}
& 75.68 \tiny{$\pm 0.40$} & 74.97 \tiny{$\pm 0.38$}
\\
& GraphENS + TAM 
& \underline{78.79} \tiny{$\pm 0.41$} & \underline{77.07} \tiny{$\pm 0.42$}
& \underline{66.83} \tiny{$\pm 0.27$} & \underline{66.32} \tiny{$\pm 0.30$}
& \underline{76.28} \tiny{$\pm 0.46$} & \underline{75.71} \tiny{$\pm 0.47$}
\\
& UPL 
& \textbf{80.12} \tiny{$\pm 0.26$} & \textbf{78.67} \tiny{$\pm 0.31$}
& \textbf{68.98} \tiny{$\pm 0.20$} & \textbf{68.62} \tiny{$\pm 0.19$}
& \textbf{76.90} \tiny{$\pm 0.30$} & \textbf{76.45} \tiny{$\pm 0.26$}
\\
\bottomrule

\end{tabular}
}

\end{scriptsize}
\end{center}
\vspace{-0.1in}
\end{table}
\begin{table}[ht!]
\caption{Experimental results of our algorithm (UPL) compared to baselines for $\rho=2$. We report the averaged balanced accuracy, averaged F1-score and the standard error of each experiment for Chameleon, Squirrel and Wisconsin datasets. Same as other experiments, all the results have been calculated over 10 repetitions and 10 pseudo-labeling iterations.}
\label{tb:Hetero_ir_2}
\begin{center}
\begin{scriptsize}
\setlength{\columnsep}{1pt}%
\resizebox{0.9\linewidth}{!}{
\begin{tabular}{@{\extracolsep{1pt}}rlcc|cc|cc@{}}
\toprule 
 & \multirow{1}{*}{\textbf{Datasets}} & \multicolumn{2}{c}{Chameleon} & \multicolumn{2}{c}{Squirrel} & \multicolumn{2}{c}{Wisconsin}  \\ 
\cline{2-8} \rule{0pt}{2.2ex}
& \textbf{Imbalance Ratio ($\rho=2$)} & bAcc. & F1 & bAcc. & F1 & bAcc. & F1 \\
\cline{2-8}
\rule{0pt}{2.5ex}  
\multirow{4}{*}{\rotatebox{90}{GCN}}

& Vanilla 
& 66.95 \tiny{$\pm 0.73$} & 66.67 \tiny{$\pm 0.69$}
& 49.53 \tiny{$\pm 0.38$} & 49.59 \tiny{$\pm 0.38$}
& 39.63 \tiny{$\pm 2.14$} & 38.28 \tiny{$\pm 2.26$}
\\
& BalancedSoftmax 
& \underline{67.10} \tiny{$\pm 0.79$} & \underline{66.78} \tiny{$\pm 0.73$}
& \underline{50.00} \tiny{$\pm 0.47$} & \underline{49.74} \tiny{$\pm 0.48$}
& \textbf{44.81} \tiny{$\pm 3.18$} & \textbf{41.32} \tiny{$\pm 2.32$}
\\
& Graphens + TAM 
& 43.76 \tiny{$\pm 1.08$} & 41.49 \tiny{$\pm 1.38$}
& 26.98 \tiny{$\pm 0.79$} & 23.69 \tiny{$\pm 0.67$}
& 43.86 \tiny{$\pm 2.89$} & 37.68 \tiny{$\pm 2.15$}
\\
& UPL 
& \textbf{67.58} \tiny{$\pm 0.56$} & \textbf{67.37} \tiny{$\pm 0.54$}
& \textbf{51.26} \tiny{$\pm 0.62$} & \textbf{51.01} \tiny{$\pm 0.56$}
& \underline{44.32} \tiny{$\pm 3.27$} & \underline{40.30} \tiny{$\pm 2.41$}
\\
\cline{2-8}
\noalign{\vskip\doublerulesep
         \vskip-\arrayrulewidth} \cline{2-8}
\rule{0pt}{2.5ex}  
\multirow{4}{*}{\rotatebox{90}{GAT}}
& Vanilla 
& \underline{63.44} \tiny{$\pm 0.50$} & \underline{63.06} \tiny{$\pm 0.45$}
& 39.20 \tiny{$\pm 1.24$} & 38.06 \tiny{$\pm 1.55$}
& 34.40 \tiny{$\pm 2.06$} & 32.28 \tiny{$\pm 1.72$}
\\
& BalancedSoftmax 
& 62.81 \tiny{$\pm 0.93$} & 62.39 \tiny{$\pm 0.89$}
& \underline{43.99} \tiny{$\pm 1.95$} & \underline{43.55} \tiny{$\pm 1.98$}
& 42.17 \tiny{$\pm 2.41$} & \underline{37.66} \tiny{$\pm 2.74$}
\\
& Graphens + TAM 
& 53.21 \tiny{$\pm 0.84$} & 52.86 \tiny{$\pm 0.93$}
& 33.21 \tiny{$\pm 0.58$} & 32.81 \tiny{$\pm 0.65$}
& \underline{44.02} \tiny{$\pm 2.16$} & 36.57 \tiny{$\pm 1.60$}
\\
& UPL 
& \textbf{65.68} \tiny{$\pm 0.57$} & \textbf{65.39} \tiny{$\pm 0.55$}
& \textbf{50.13} \tiny{$\pm 0.39$} & \textbf{49.68} \tiny{$\pm 0.35$}
& \textbf{44.40} \tiny{$\pm 2.51$} & \textbf{41.34} \tiny{$\pm 2.44$}
\\

\cline{2-8}
\noalign{\vskip\doublerulesep
         \vskip-\arrayrulewidth} \cline{2-8}
\rule{0pt}{2.5ex}  
\multirow{4}{*}{\rotatebox{90}{Sage}}
& Vanilla 
& \underline{59.18} \tiny{$\pm 0.42$} & \underline{58.85} \tiny{$\pm 0.41$}
& 36.10 \tiny{$\pm 0.63$} & 35.98 \tiny{$\pm 0.63$}
& 65.57 \tiny{$\pm 2.81$} & 66.42 \tiny{$\pm 2.48$}
\\
& BalancedSoftmax 
& 59.17 \tiny{$\pm 0.52$} & 58.55 \tiny{$\pm 0.50$}
& \underline{37.22} \tiny{$\pm 0.53$} & \textbf{37.00} \tiny{$\pm 0.52$}
& \underline{74.10} \tiny{$\pm 2.77$} & \underline{70.37} \tiny{$\pm 3.04$}
\\
& Graphens + TAM 
& 51.25 \tiny{$\pm 0.90$} & 50.87 \tiny{$\pm 0.93$}
& 32.86 \tiny{$\pm 0.49$} & 32.52 \tiny{$\pm 0.51$}
& 63.35 \tiny{$\pm 3.53$} & 53.89 \tiny{$\pm 3.07$}
\\
& UPL 
& \textbf{60.01} \tiny{$\pm 0.56$} & \textbf{59.59} \tiny{$\pm 0.57$}
& \textbf{37.30} \tiny{$\pm 0.58$} & \underline{36.94} \tiny{$\pm 0.58$}
& \textbf{75.21} \tiny{$\pm 3.21$} & \textbf{73.20} \tiny{$\pm 2.86$}
\\

\bottomrule

\end{tabular}
}

\end{scriptsize}
\end{center}
\vspace{-0.1in}
\end{table}
\clearpage

\subsection{Additional Datasets}
To further test UPL, we have included additional experiments with Amazon datasets \citep{shchur2018pitfalls}, comparing them to other previous works in Table~\ref{tb:Amazon}, and ablation studies in Table \ref{tb:Amazon_ablation}.

\begin{table}[ht!]
\caption{Experimental results of our algorithm (UPL) compared to baselines for $\rho=10$. We report the averaged balanced accuracy, averaged F1-score and the standard error of each experiment for Computers and Photo datasets. All the results have been calculated over 10 repetitions and 10 pseudo-labeling iterations.}
\begin{center}
\begin{scriptsize}
\setlength{\columnsep}{1pt}%
\resizebox{0.9\linewidth}{!}{
\begin{tabular}{@{\extracolsep{1pt}}rlcc|cc@{}}
\toprule 
 & \multirow{1}{*}{Datasets} & \multicolumn{2}{c}{Computers} & \multicolumn{2}{c}{Photo}  \\ 
\cline{2-6} \rule{0pt}{2.2ex}
& \textbf{Imbalance Ratio ($\rho=10$)} & bAcc. & F1 & bAcc. & F1 \\
\cline{2-6}
\rule{0pt}{2.5ex}  
\multirow{4}{*}{\rotatebox{90}{GCN}} 
                    \rule{0pt}{2ex}
& Vanilla 
& 47.11 \tiny{$\pm 2.28$} & 36.53 \tiny{$\pm 3.27$}
& 53.80 \tiny{$\pm 3.09$} & 44.83 \tiny{$\pm 4.50$}
\\
& BalancedSoftmax 
& \underline{74.35} \tiny{$\pm 1.30$} & \underline{73.73} \tiny{$\pm 1.35$}
& 74.86 \tiny{$\pm 1.20$} & 74.33 \tiny{$\pm 1.39$}
\\
& GraphENS + TAM 
& 73.82 \tiny{$\pm 2.76$} & 72.68 \tiny{$\pm 3.02$}
& \textbf{82.34} \tiny{$\pm 1.76$} & \textbf{81.39} \tiny{$\pm 2.26$}
\\
& UPL 
& \textbf{75.53} \tiny{$\pm 0.94$} & \textbf{75.84} \tiny{$\pm 0.97$}
& \underline{79.27} \tiny{$\pm 0.93$} & \underline{78.67} \tiny{$\pm 0.93$}
\\
\noalign{\vskip\doublerulesep
         \vskip-\arrayrulewidth} \cline{2-6}
\rule{0pt}{2.5ex}  
\multirow{4}{*}{\rotatebox{90}{GAT}}
& Vanilla 
& 41.24 \tiny{$\pm 1.24$} & 32.45 \tiny{$\pm 1.92$}
& 46.52 \tiny{$\pm 0.67$} & 37.62 \tiny{$\pm 1.55$}
\\
& BalancedSoftmax 
& \underline{64.11} \tiny{$\pm 2.09$} & \underline{63.01} \tiny{$\pm 2.70$}
& \underline{71.36} \tiny{$\pm 1.90$} & \underline{69.78} \tiny{$\pm 2.07$}
\\
& GraphENS + TAM 
& 64.55 \tiny{$\pm 0.95$} & 62.59 \tiny{$\pm 1.34$}
& 68.86 \tiny{$\pm 1.34$} & 65.52 \tiny{$\pm 1.60$}
\\
& UPL 
& \textbf{74.62} \tiny{$\pm 1.08$} & \textbf{74.20} \tiny{$\pm 1.10$}
& \textbf{81.93} \tiny{$\pm 0.87$} & \textbf{81.60} \tiny{$\pm 0.91$}
\\
\noalign{\vskip\doublerulesep
         \vskip-\arrayrulewidth} \cline{2-6}
\rule{0pt}{2.5ex}  
\multirow{4}{*}{\rotatebox{90}{Sage}} 
& Vanilla 
& 48.16 \tiny{$\pm 1.84$} & 38.56 \tiny{$\pm 2.70$}
& 60.43 \tiny{$\pm 2.83$} & 52.07 \tiny{$\pm 3.69$}
\\
& BalancedSoftmax 
& 73.18 \tiny{$\pm 1.51$} & 72.19 \tiny{$\pm 1.71$}
& 78.55 \tiny{$\pm 1.41$} & 77.46 \tiny{$\pm 1.59$}
\\
& GraphENS + TAM 
& \underline{75.20} \tiny{$\pm 1.47$} & \underline{73.96} \tiny{$\pm 1.72$}
& \underline{81.36} \tiny{$\pm 1.54$} & \underline{80.67} \tiny{$\pm 1.64$}
\\
& UPL 
& \textbf{81.56} \tiny{$\pm 0.94$} & \textbf{81.49} \tiny{$\pm 1.01$}
& \textbf{89.00} \tiny{$\pm 0.44$} & \textbf{88.90} \tiny{$\pm 0.46$}
\\

\bottomrule

\end{tabular}
}

\end{scriptsize}
\end{center}
\label{tb:Amazon}
\vspace{-0.1in}
\end{table}
\begin{table}[ht!]
\caption{We conducted an ablation study on our UPL algorithm, comparing it to baselines by systematically removing key components. We report the averaged balanced accuracy, averaged F1-score, and the standard error for each experiment on the Computers and Photo datasets. The results were calculated over 10 repetitions and 10 pseudo-labeling iterations.}
\begin{center}
\begin{scriptsize}
\setlength{\columnsep}{1pt}%
\resizebox{0.9\linewidth}{!}{
\begin{tabular}{@{\extracolsep{1pt}}rlcc|cc@{}}
\toprule 
 & \multirow{1}{*}{Datasets} & \multicolumn{2}{c}{Computers} & \multicolumn{2}{c}{Photo}  \\ 
\cline{2-6} \rule{0pt}{2.2ex}
& \textbf{Imbalance Ratio ($\rho=10$)} & bAcc. & F1 & bAcc. & F1 \\
\cline{2-6}
\rule{0pt}{2.5ex}  
\multirow{8}{*}{\rotatebox{90}{GCN}} 
                    \rule{0pt}{2ex}
                     & Vanilla 
& 46.15 \tiny{$\pm 1.61$} & 35.25 \tiny{$\pm 2.37$}
& 48.30 \tiny{$\pm 2.46$} & 36.86 \tiny{$\pm 3.46$}
\\                   
                     & + BS
& 74.24 \tiny{$\pm 0.92$} & \underline{74.57} \tiny{$\pm 0.92$}
& \underline{76.86} \tiny{$\pm 0.87$} & \underline{75.97} \tiny{$\pm 0.92$}
\\

                     & + PL
& 61.95 \tiny{$\pm 1.55$} & 57.85 \tiny{$\pm 1.92$}
& 63.95 \tiny{$\pm 2.97$} & 58.74 \tiny{$\pm 4.17$}
\\
 
                     & + Uncertainty
& 47.96 \tiny{$\pm 2.25$} & 37.78 \tiny{$\pm 3.19$}
& 48.25 \tiny{$\pm 2.42$} & 36.79 \tiny{$\pm 3.40$}
\\

                     & + BS + PL
& 73.56 \tiny{$\pm 1.17$} & 73.74 \tiny{$\pm 1.14$}
& 76.86 \tiny{$\pm 0.99$} & 75.92 \tiny{$\pm 1.08$}
\\

                     & + Uncertainty + BS 
& \underline{74.44} \tiny{$\pm 1.02$} & 74.54 \tiny{$\pm 1.07$}
& 76.61 \tiny{$\pm 0.92$} & 75.91 \tiny{$\pm 0.96$}
\\

                     & + Uncertainty + PL
& 62.20 \tiny{$\pm 3.10$} & 55.43 \tiny{$\pm 3.83$}
& 72.14 \tiny{$\pm 2.93$} & 68.49 \tiny{$\pm 3.95$}
\\
                     
                     & UPL
& \textbf{75.53} \tiny{$\pm 0.94$} & \textbf{75.84} \tiny{$\pm 0.97$}
& \textbf{79.27} \tiny{$\pm 0.93$} & \textbf{78.67} \tiny{$\pm 0.93$}
\\

\bottomrule

\end{tabular}
}

\end{scriptsize}
\end{center}
\label{tb:Amazon_ablation}
\vspace{-0.1in}
\end{table}

\subsection{Discussion}

Both Table \ref{tb:apendix_homo} and Table \ref{tb:apendix_hetero} demonstrate a clear superiority of our method across various architectures and datasets. A particularly notable aspect of our experiments is the significant margin by which our results surpass the second-best outcomes. This pronounced difference underscores the effectiveness of our approach, setting a new benchmark for performance in this field.

Our experimental analysis investigated the impact of varying the number of edge removal iterations ($t$) and the number of edges removed in each iteration ($S_k$) on model performance. The results, provided in Figure \ref{fig:augmentation_sweep}, demonstrate that choosing both $t=100$ and $S_k=100$ yields the best results in both studied datasets. Notably, the computation of uncertainty values does not require any training steps, which significantly reduces the computational overhead. Given that typical model training procedures require over 1000 epochs, the time needed to calculate uncertainty values is negligible in comparison to the overall training process.
\begin{figure}[ h!]
    \centering
    \includegraphics[width=0.6\columnwidth]{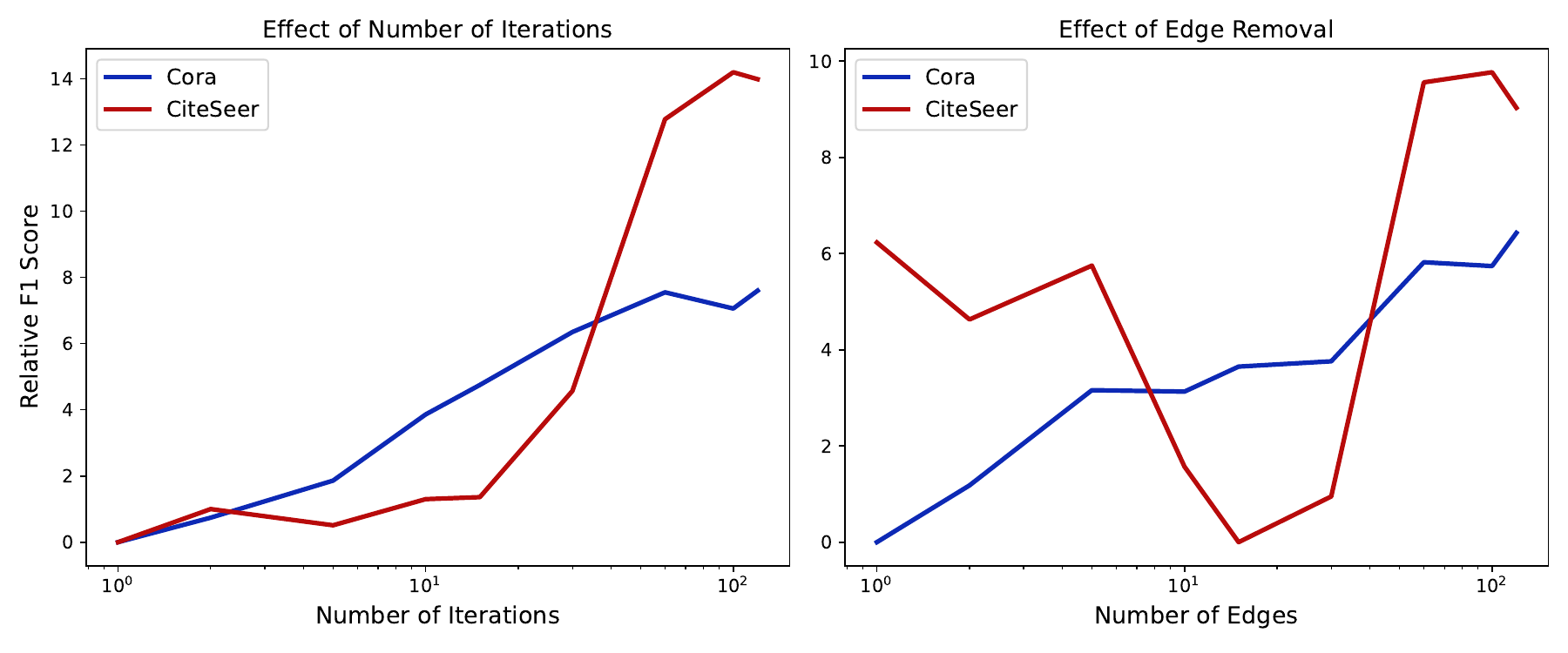}
    \caption{Selection Edge Removal: F1 score versus number of iterations and number of edges for removal. Each plot is normalized to it's lowest value.}
    \label{fig:augmentation_sweep}
\end{figure}
\begin{table}[ht!]
\caption{UPL performance summary: Comparison of UPL in terms of the number of \textcolor{red}{Best}\textbar Second rank performance.}
\begin{center}
\begin{scriptsize}
\setlength{\columnsep}{1pt}%
\resizebox{0.4\linewidth}{!}{
\begin{tabular}{@{\extracolsep{1pt}}lc|c|c@{}}
\toprule 
Datasets Types & GCN & GAT & SAGE  \\ 
\cline{1-4}
Homophilic & \textcolor{red}{$9$}$|2$ & \textcolor{red}{$10$}$|1$ & \textcolor{red}{$11$}$|0$ 
\\
Heterophilic&\textcolor{red}{$7$}$|2$ & \textcolor{red}{$9$}$|0$ & \textcolor{red}{$7$}$|2$
\\
\cdashline{1-4}
Total & \textcolor{red}{$16$}$|4$ & \textcolor{red}{$19$}$|1$ & \textcolor{red}{$18$}$|2$ \\
\bottomrule

\end{tabular}
}

\end{scriptsize}
\end{center}
\label{tb:performance_comparison}
\vspace{-0.1in}
\end{table}

For a summary of the performance of UPL, We have provided Table \ref{tb:performance_comparison}, Comparing our model with previous methods in term of F1 score ranking based on topological nature of the dataset and the employed backbone. Performance ranking is measured in terms of \textcolor{red}{Best}\textbar Second, It is worth noting that this ranking is adequate as UPL hasn't performed lower than second compared to previous methods in any of the experiments
\section{Detailed Comparison with UNREAL}\label{app:comp_unreal}

To assess practical efficiency, we conducted time measurements for both UNREAL and UPL under identical experimental conditions. Our findings demonstrate that UPL’s streamlined approach, which employs fewer and less complex iterative processes, substantially reduces both time and computational resource requirements, resulting in superior empirical efficiency. This advantage is clearly illustrated in our experiments with the PubMed dataset: UPL completed 10 runs in an average runtime of 0.66 minutes, while UNREAL required an average runtime of 14.1 minutes for the same number of runs. These timing tests were performed on a server equipped with an Nvidia RTX 3090 GPU and AMD Ryzen 5 3600x CPU @ 3.80GHz, ensuring consistent and comparable results. While UPL uses less time to run, it also delivers better performance, as shown in Tables \ref{tb:main_chart_homo}, \ref{tb:main_chart_hetero}, \ref{tb:apendix_homo}, and \ref{tb:apendix_hetero}. Our extensive evaluation across 24 different experimental setups reveals that UPL consistently outperforms UNREAL. In 17 out of these 24 setups, UPL achieved better results in both balanced accuracy and F1 score, while in the remaining cases, it secured second place, highlighting its robustness and consistent performance across a variety of conditions.

\end{document}